\documentclass{article}

\usepackage{fullpage}

\usepackage{natbib}
\usepackage[utf8]{inputenc}
\usepackage{amsmath,amsfonts}
\usepackage{amsthm}
\usepackage{amssymb}
\usepackage{bm}
\usepackage{bbm}
\usepackage{graphicx}
\usepackage{floatrow}
\newfloatcommand{capbtabbox}{table}[][\FBwidth]

\usepackage[utf8]{inputenc} 
\usepackage[T1]{fontenc}    
\usepackage{lmodern}
\usepackage[hidelinks]{hyperref}       
\usepackage{url}            
\usepackage{booktabs}       
\usepackage{amsfonts}       
\usepackage{nicefrac}       
\usepackage{microtype}      
\usepackage{xcolor}         
\usepackage{enumitem}

\usepackage{algorithm}
\usepackage[noend]{algpseudocode}

\DeclareMathOperator{\E}{\mathbb E}

\DeclareMathOperator{\bP}{\mathbb P}

\DeclareBoldMathCommand{\w}{w}
\DeclareBoldMathCommand{\u}{u}
\DeclareBoldMathCommand{\W}{W}
\DeclareBoldMathCommand{\U}{U}
\DeclareBoldMathCommand{\s}{s}
\DeclareBoldMathCommand{\x}{x}
\DeclareBoldMathCommand{\1}{1}
\DeclareBoldMathCommand{\0}{0}
\DeclareBoldMathCommand{\g}{g}
\DeclareBoldMathCommand{\z}{z}
\DeclareBoldMathCommand{\e}{e}
\DeclareBoldMathCommand{\p}{p}
\DeclareBoldMathCommand{\y}{y}

\newcommand{\reals}{\mathbb{R}}
\newcommand\doma[1]{\mathcal{#1}}
\newcommand{\Sset}{\mathcal{S}}
\newcommand{\Qset}{\mathcal{Q}}

\newcommand{\Aset}{\mathcal{A}}
\newcommand{\Eset}{\mathcal{E}}
\newcommand{\Vset}{\mathcal{V}}
\newcommand{\Gset}{\mathcal{G}}
\newcommand{\Mset}{\mathcal{M}}

\newcommand{\regret}{\mathcal{R}}
\newcommand{\sumT}{\sum_{t=1}^T}
\newcommand{\sumt}{\sum_{j=1}^t}

\newcommand{\ystar}{y^\star}
\newcommand{\ystart}{y_t^\star}
\newcommand{\mstart}{m_t^\star}
\newcommand{\lhatt}{\widehat{\ell}_t}

\newcommand{\id}{\mathbbm{1}}

\newcommand{\half}{\tfrac{1}{2}}
\newcommand\inner[2]{\langle #1, #2 \rangle}
\newcommand\Eb[1]{\E\left[ #1\right]}
\newcommand\Ob[1]{O\left( #1\right)}

\newcommand{\zot}{{\id[y_t' \not = y_t]}}
\newcommand{\zopt}{{\sum_{y \in [K]} p_t'(y) \id[y \not = y_t]}}

\newcommand{\vmax}{v_{\textnormal{max}}}
\newcommand{\Vmax}{V_{\textnormal{max}}}

\newcommand{\lmax}{\ell_{\textnormal{max}}}

\newcommand{\lhat}{\widehat{\ell}}
\newcommand{\ghat}{\widehat{\g}}
\newcommand{\vt}{{v}_t}
\newcommand{\revt}{\big|\{t: y^\star_t \not \in \Qset\}\big|}

\DeclareMathOperator*{\argmax}{arg\,max}

\newcommand{\bUhat}{\widehat{\U}}
\newcommand{\Uhat}{\widehat{U}}
\newcommand{\norm}[1]{\big\|{#1}\big\|}
\newcommand{\theset}[2]{\left\{{#1}\,:\,{#2}\right\}}
\newcommand{\dt}{\displaystyle}

\DeclareMathOperator{\myvec}{vec}

\newtheorem{theorem}{Theorem}
\newtheorem{lemma}{Lemma}

\newtheorem{corollary}{Corollary}

\usepackage{thmtools, thm-restate}

\title{
Beyond Bandit Feedback \\ in Online Multiclass Classification
}

\author{%
  Dirk van der Hoeven \\
  Dept.\ of Computer Science \\
  Università degli Studi di Milano, Italy \\
  \and
  Federico Fusco\\
  Dept.\ of Computer, Control and Management Engineering \\
  Sapienza Università di Roma, Italy\\
  \and
  Nicolò Cesa-Bianchi \\
  DSRC \& Dept.\ of Computer Science \\
  Università degli Studi di Milano, Italy
}
\date{}
\newcommand{\todo}[1]{%
\ifmmode
\text{\textcolor{red}{TODO: #1}}
\else
\textcolor{red}{TODO: #1}
\fi
}

\newcommand{\Dirk}[1]{%
\ifmmode
\text{\textcolor{red}{Dirk: #1}}
\else
\textcolor{red}{Dirk: #1}
\fi
}

\newcommand{\NCB}[1]{%
\ifmmode
\text{\textcolor{blue}{NCB: #1}}
\else
\textcolor{blue}{NCB: #1}
\fi
}

\newcommand{\fnote}[1]{%
\ifmmode
\text{\textcolor{green}{FF: #1}}
\else
\textcolor{green}{FF: #1}
\fi
}

\begin{document}

\maketitle

\begin{abstract}
We study the problem of online multiclass classification in a setting where the learner's feedback is determined by an arbitrary directed graph. While including bandit feedback as a special case, feedback graphs allow a much richer set of applications, including filtering and label efficient classification.
We introduce \textproc{Gappletron}, the first online multiclass algorithm that works with arbitrary feedback graphs. For this new algorithm,
we prove surrogate regret bounds that hold, both in expectation and with high probability, for a large class of surrogate losses. Our bounds are of order $B\sqrt{\rho KT}$, where $B$ is the diameter of the prediction space, $K$ is the number of classes, $T$ is the time horizon, and $\rho$ is the domination number (a graph-theoretic parameter affecting the amount of exploration). In the full information case, we show that \textproc{Gappletron} achieves a constant surrogate regret of order $B^2K$. We also prove a general lower bound of order $\max\big\{B^2K,\sqrt{T}\big\}$ showing that our upper bounds are not significantly improvable. Experiments on synthetic data show that for various feedback graphs 
our algorithm is competitive against known baselines.
\end{abstract}

\section{Introduction}

In online multiclass classification a learner interacts with an unknown 
environment in a sequence of rounds. At each round $t$, the learner observes a feature vector $\x_t \in \reals^d$ and outputs a prediction $y_t'$ for the label $y_t \in \{1,\ldots,K\}$ associated with $\x_t$.
If $y_t' \neq y_t$, then the learner is charged with a mistake.
\citet{kakade2008efficient} introduced the bandit version of online multiclass classification, where the only feedback received by the learner after each prediction is the loss $\id[y_t' \not = y_t]$. Hence, if a mistake is made at time $t$ (and $K > 2$), the learner cannot uniquely identify the true label $y_t$ based on the feedback information.

Although bandits are a canonical example of partial feedback, they fail to capture a number of important practical scenarios of online classification
. Consider for example spam filtering, where an online learner is to classify emails as spam or non-spam based on their content. Whenever the learner classifies an email as legitimate, the recipient gets to see it, and can inform the learner whether the email was correctly classified of not. However, when the email is classified as spam, the learner does not get any feedback because the email is not checked by the recipient. Another example is label efficient multiclass classification. Here, instead of making a prediction, the learner can ask a human expert for the true label.
At the steps when predictions are made, however, the learner does not receive any feedback information (not even their own loss).
A further example is disease prevention: if we predict an outburst of disease in a certain area, we can preemptively stop it by vaccinating the local population. This intervention would prevent us from observing whether our prediction was correct, but would still allow us to observe an outburst occurring in a different area.

Unlike bandits, the amount of feedback obtained by the learner in these examples depends on the predicted class, and can vary from full information to no feedback at all. This scenario has been previously considered in the framework of online learning with feedback graphs \citep{Mannor2011from,AlonCDK15,alon2017nonstochastic}.
A feedback graph is a directed graph $\Gset = (\Vset, \Eset)$ where each node in $\Vset$ receives at least one edge from some other node in $\Vset$ (possibly from itself). The nodes in $\Vset$ correspond to actions, and a directed edge $(a,b) \in \Eset$ indicates that by playing action $a$ the learner observes the loss of action $b$. This generalizes the well-known online learning settings of experts (where $\Gset$ is the complete graph, including self-loops) and bandits (where $\Gset$ has only self-loops). Note that it is easy to model spam filtering and label efficient prediction using feedback graphs. For spam filtering, $\Gset$
contains only two actions $s$ and $n$ (corresponding, respectively, to the learner's predictions for spam and non-spam), and the edge set is $\Eset = \big\{(n, n), (n, s)\big\}$. For label efficient multiclass prediction, $\Gset$ contains a node for each class, plus an extra node corresponding to issuing a label request.
It is important to observe that all previous analyses of feedback graphs only apply to the abstract setting of prediction with experts, where any dependence of the loss on feature vectors is ignored. This hampers the application of those results to online multiclass classification.
In this work we build on previous results on online learning and classification with bandit feedback to design and analyze the first algorithm for online multiclass classification with arbitrary feedback graphs. In doing so, we also improve the analyses of the previously studied special cases (full information and bandit feedback) of multiclass classification.

In the online multiclass classification setting, the goal is bound the number of mistakes made by the learner. The mistake bounds take the following form:
\begin{align}\label{eq:surrogate regret def}
    \sumT \zot = \sumT \ell_t(\U) + \regret_T,
\end{align}
where $\ell_t$ is a surrogate loss, $\U \in \doma{W} \subseteq \reals^{d \times K}$ is the matrix of reference predictors, and $\regret_T$ is called the surrogate regret. In this work we provide two types of bounds on the %
surrogate regret:
bounds that hold in expectation and bounds that hold with high probability. Note that equation \eqref{eq:surrogate regret def} could also be written as $\sumT \big(\zot - \ell_t(\U)\big) = \regret_T$. However, we prefer the former former since $\regret_T$ is not a proper regret: because the zero-one loss is non-convex we compare it with a surrogate loss.

Our results build on recent work by \citet{vanderHoeven2020exploiting}, who showed that one can exploit the gap between the surrogate loss and the zero-one loss to derive improved surrogate regret bounds
in the full information and bandit settings of online multiclass classification. We modify the \textproc{Gaptron} algorithm \citep{vanderHoeven2020exploiting} to make it applicable to the feedback graph setting. In the analysis of the resulting algorithm, called \textproc{Gappletron}\footnote{Our algorithm is called after the apple tasting feedback model, which is the original name of the spam filtering graph.}, we use several new insights to show that it has %
$O(B\sqrt{\rho KT})$ surrogate regret in expectation and $O\big(\!\sqrt{\rho KT(B^2 + \ln(1/\delta))}\big)$ surrogate regret with probability at least $1-\delta$ for any feedback graph with domination number\footnote{The domination number is the cardinality of the smallest dominating set.} $\rho$, and for any $\norm{\myvec(\U)} \leq B$ for some norm $\norm{\cdot}$
(if $\norm{\cdot}$ is the Euclidean norm, then $\norm{\myvec(\U)}$ is the Frobenius norm of $\U$). For example, in both spam filtering and label efficient classification we have $\rho=1$.
So in the label efficient setting, where each label request counts as a mistake, with high probability \textproc{Gappletron} makes at most order of $B\sqrt{KT}$ mistakes while requesting at most order of $B\sqrt{KT}$ labels. Note that we are not aware of previously known high-probability bounds on the surrogate regret. Furthermore, whereas the results of \citet{vanderHoeven2020exploiting} only hold for a limited number of surrogate loss functions, our results hold for the larger class of regular surrogate losses.

Interestingly, with feedback graphs the surrogate regret for online multiclass classification has, in general, a better dependence on $T$ than the regret for online learning. Indeed, \citet{AlonCDK15} show that the best possible online learning regret is $\Omega(T^{2/3})$ for certain feedback graphs that are called weakly observable (e.g., the graphs for label efficient classification). In contrast, we prove a $O(T^{1/2})$ upper bound on the surrogate regret for any feedback graph, including weakly observable ones.

Our results cannot be significantly improved in general: we prove a $\Omega(B^2K + \sqrt{T})$ lower bound on the surrogate regret. Due to the new insights required by their proofs, we believe the high-probability bounds and the lower bounds are our strongest technical contributions.
 
We provide several other new results. In the separable case, when there exists a $\U$ for which $\sumT \ell_t(\U) = 0$, \textproc{Gappletron} has $O(B\sqrt{\rho T})$ surrogate regret in expectation. Even though $O(B^2K)$ mistake bounds are possible in the separable setting \citep{beygelzimer2019bandit}, ours is the first algorithm that has satisfactory surrogate regret in the non-separable case and has an improved surrogate regret in the separable case. Note that although \textproc{Banditron} \citep{kakade2008efficient} also makes $O(B\sqrt{KT})$ mistakes in the separable case, it suffers $O(K^{1/3}(BT)^{2/3})$ surrogate regret in the non-separable case. Our results for the separable case in the full information setting improve results of \citet{vanderHoeven2020exploiting} by a factor of $K$: \textproc{Gappletron} suffers $O(B^2)$ surrogate regret both in expectation and with high probability, thus matching the bounds of the classical \textproc{Perceptron} algorithm \citep{rosenblatt1958perceptron, Novikov1962}  --- see Table~\ref{table:intro} for a summary of our theoretical results. Finally, we also evaluated the performance of \textproc{Gappletron} in several experiments, showing that \textproc{Gappletron} is competitive against known baselines in the full information, bandit, and multiclass spam filtering setting, in which predicting a certain class provides full information feedback and all other predictions do not provide any information %
(see Figure~\ref{fig:intro} for an experimental result in the bandit setting).
\begin{figure}
\begin{floatrow}
\capbtabbox{%
    \begin{tabular}{lccc}
    \toprule
    \textbf{Upper bounds}  & \quad \quad & Partial & Full \\
    \midrule
    Non-separable & \quad  & $B \sqrt{\rho KT}$ & $KB^2$ \\
    Separable & \quad & $B \sqrt{\rho T}$ & $B^2$ \\ \midrule\midrule
    \textbf{Lower bounds}  & \quad & \ & \ \\ \midrule
    Non-separable &\quad  & \ & $KB^2$ \\
    Separable &\quad  & $\sqrt{T}$ & $B^2$\ \\
    \bottomrule
    \end{tabular}
}{%
  \caption{
  \label{table:intro}
  \small{Overview of the surrogate regret bounds in the separable and non-separable case. %
  The upper bounds hold with high probability, while the lower bounds apply to any randomized prediction algorithm. All bounds are novel except for the lower bound in the full information separable case \citep[Theorem 11]{beygelzimer2019bandit}}.
  }%
}
\hspace{2mm}
\ffigbox{%
  \includegraphics[width = 0.48\textwidth]{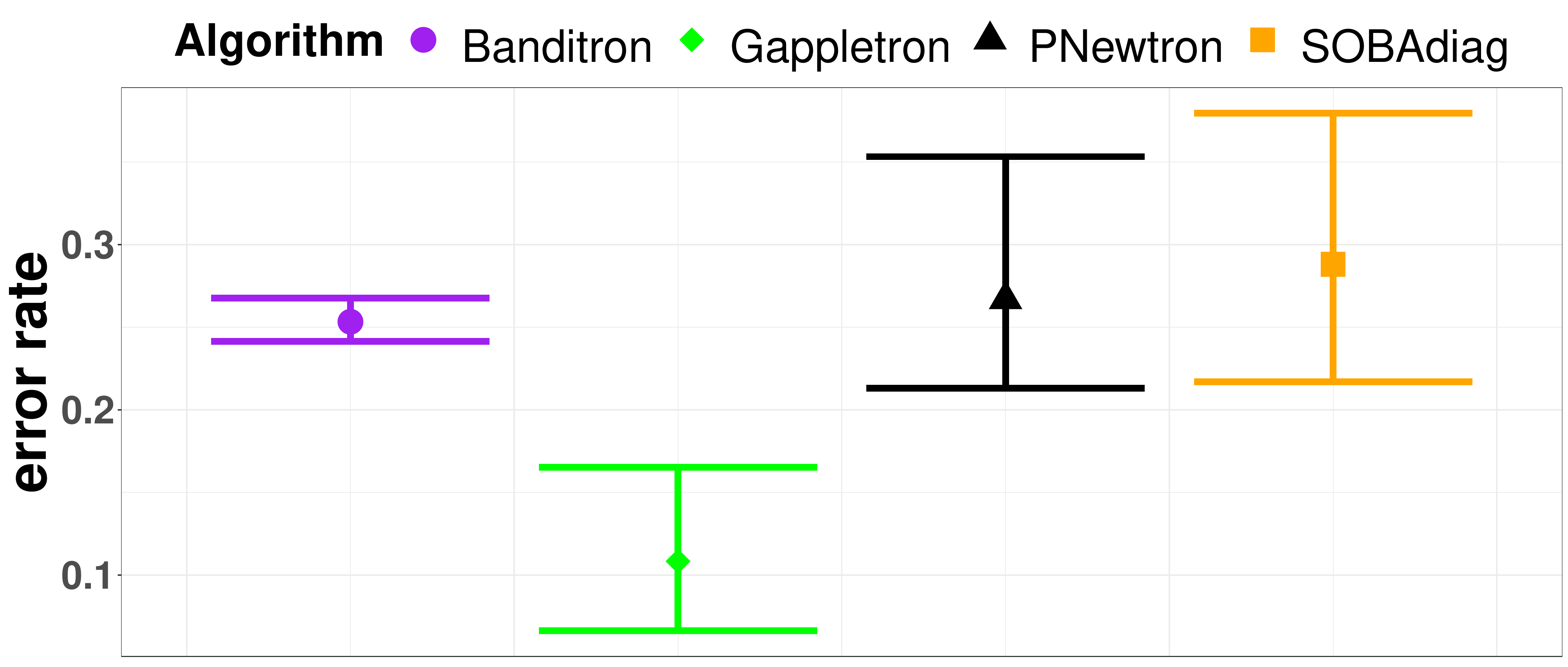}%
}{%
  \caption{
  \label{fig:intro}
  \small{Error rate in non-separable synthetic bandit experiments showcasing \textproc{Gappletron} against known baselines. The points are the means and the whiskers are minimum and maximum error rate over ten repetitions (details in Section \ref{sec:experiments}).
  }}%
}
\end{floatrow}
\end{figure}
\paragraph{Additional related work.}
The full information and bandit versions of the online multiclass classification setting have been extensively studied. Here we provide the most relevant references and defer the reader to \citet{vanderHoeven2020exploiting}  
for a more extensive literature review. Algorithms for the full information setting include: the \textproc{Perceptron}, its multiclass versions \citep{rosenblatt1958perceptron, crammer2003ultraconservative, fink2006online} and many variations thereof, second-order algorithms such as \textproc{Arow} \citep{crammer2009adaptive} and the second-order \textproc{Perceptron} \citep{cesa2005second}, and various algorithms for online logistic regression --- see \citet{foster2018logistic} and references therein. In the bandit setting, we mention the algorithms  \textproc{Banditron} \citep{kakade2008efficient}, \textproc{Newtron} \citep{hazan2011newtron},
\textproc{SOBA} \citep{beygelzimer2017efficient}, and \textproc{Obama} \citep{foster2018logistic}.

Online learning with feedback graphs has been investigated both in the adversarial and stochastic regimes. In the adversarial setting, variants where the graph changes over time and is partially known or stochastic have been studied by \citet{cohen2016online,kocak2016online}. Regret bounds that scale with the loss of the best action have been obtained by \citet{lykouris2018small}. Other variants include sleeping experts \citep{cortes2019online}, switching experts \citep{arora2019bandits}, and adaptive adversaries \citep{feng2018online}. Some works use feedback graphs to bound the regret in auctions \citep{cesa2017algorithmic,feng2018learning,han2020learning}. In the stochastic setting, regret bounds for Thompson sampling and UCB have been analyzed by \citet{tossou2017thompson,liu2018analysis,lykouris2020feedback}. Finally, feedbacks graphs can also be viewed as a special case of the partial monitoring framework for sequential decisions, see  \citep{lattimore2018bandit} for an introduction to the area.

\citet{helmbold2000apple} introduced online filtering as ``apple tasting''. However, their analysis applies to a restricted version of online learning in which instances $x_t$ belong to a finite domain, and the labels $y_t$ are such that $y_t = f(x_t)$ for all $t$ and for some fixed $f$ in a known class of functions. Practical applications of online learning to spam filtering have been investigated by \citet{cesa2003margin,sculley2008advances}. 

\paragraph{Notation.} Let $\1$ and $\0$ denote, respectively, the all-one and all-zero vectors, and let $\e_k$ be the basis vector in direction $k$. Let $[K] = \{1, \ldots, K\}$ and let $\reals_+$ be the non-negative real numbers. We use $\langle \g, \w \rangle$ to denote the inner product between vectors $\g,\w \in \reals^d$. The rows of matrix $\W \in \reals^{K \times d}$ are denoted by $\W^1, \ldots, \W^K$.
Whenever possible, we use the same symbol $\W$ to denote both a $K\times d$ matrix and a column vector $\myvec(\W) = (\W^1, \ldots, \W^K)$ in $\reals^{Kd}$. We use $\|\x\|_2$ to denote the Euclidean norm of a vector $\x$ and $\|\x\|$ to denote an arbitrary norm. The Kronecker product between matrices $\W$ and $\U$ is denoted by $\W \otimes \U$. We assume $\W\in\doma{W}$ for some convex $\doma{W} \subseteq \reals^{K \times d}$. This is equivalent to say that $\myvec(\W)$ belongs to a convex subset of $\reals^{Kd}$, for example a $p$-norm ball.
As in previous works, we assume instance-label pairs $(\x_t,y_t)$ are generated by an adversary who is oblivious to the algorithm's internal randomization.
Finally, for any round $t$, $P_t[\cdot]$ and $\E_t[\cdot]$ denote the conditional probability and expectation, given the randomized predictions $y_1', y_2', \ldots, y_{t-1}'$ and the corresponding feedback.

A feedback graph is any directed graph $\Gset = (\Vset, \Eset)$, with edges $\Eset$ and nodes $\Vset$, such that for any $y\in\Vset$ there exists some $y'\in\Vset$ such that $(y',y) \in \Eset$, where we allow $y'=y$. In online multiclass classification, $\Vset = [K]$ and $\Eset$ specifies which predictions observe which outcomes. Let $\textnormal{out}(y') = \{y \in \Vset : (y',y) \in \Eset\}$ be the out-neighbourhood of $y'$. If the learner predicts $y_t'$ at time $t$, then the feedback received by the learner is the set of pairs $\big(y,\id[y \not = y_t]\big)$ for all $y \in \textnormal{out}(y')$. Due to the structure of the zero-one loss, if a node has $K-1$ outgoing edges, we always add the missing edge to $\Eset$ as this does not change the information available to the learner. We say that an outcome $y'$ is revealing if predicting that outcome provides the learner with full information feedback, i.e., $\textnormal{out}(y')=[K]$, and we denote the set of revealing outcomes by $\Qset$. For example, in label efficient classification, querying the true label $y_t$ corresponds to playing a revealing outcome. We say that a set of nodes $\Sset$ is a dominating set if for each $y \in \Vset$ there is a node $y' \in \Sset$ such that $y \in \textnormal{out}(y')$.
The number of nodes in a smallest dominating set is called the domination number, and we denote it by $\rho$. Note that \textsc{Gappletron} is run using the minimum dominating set $\Sset$, which is known to be hard to recover in general. However, if the algorithm is fed with any other dominating set $\Sset'$ of bigger cardinality $\rho'$, our results continue to hold with $\rho$ replaced by $\rho'$ (recall that a dominating set of size at most $(\ln\rho + 2)\rho$ can be efficiently found via a greedy approximation algorithm).  

\paragraph{Regular surrogate losses.}
Fix a convex domain $\doma{W}$. Let $\ell: \doma{W} \times \reals^d \times [K] \to \reals_+$ be any function convex on $\doma{W}$ such that, for all $\W \in \doma{W}$, $\x\in\reals^d$, and $y \in [K]\setminus\{\ystar\}$ (with $\ystar = \argmax_k \inner{\W^k}{\x}$) we have
\begin{equation}\label{eq:wrong plus right condition}
    \frac{K-1}{K}\,\ell(\W, \x, y) + \frac{1}{K} \ell(\W, \x, \ystar) \geq 1. 
\end{equation}
Then $\ell_t = \ell(\cdot, \x_t, y_t)$ is a regular surrogate loss if 
\begin{equation}\label{eq:gradient condition}
    \|\nabla \ell_t(\W)\|^2 \leq 2 L\,\ell_t(\W) \qquad \W\in\doma{W}
\end{equation}
for some norm $\|\cdot\|$. When $\|\cdot\|$ is the Euclidean norm, the condition on the gradient is satisfied by all $L$-smooth surrogate loss functions (see, for example, \citep[Lemma 4]{zhou2018fenchel}). 

%
%

Examples of regular surrogate losses are the smooth hinge loss \citep{rennie2005loss} and the logistic loss with base $K$, defined by $\ell_t(\W_t) = -\log_K q(\W_t, \x_t, y_t)$, where $q$ is the softmax function.  
Even though the hinge loss is not a regular surrogate loss, in Appendix \ref{app:hinge loss} we show that a particular version of the hinge loss satisfies all the relevant properties of regular surrogate losses. Also, note that in the feedback graph setting, this particular version of the hinge loss we use is random whenever the learner's predictions are randomized.

\section{Gappletron}\label{sec:gappletron}

\begin{algorithm}[h]
\caption{\textproc{Gappletron}}\label{alg:gappletron}
\begin{algorithmic}[1]
\Require Set of revealing actions $\Qset \subseteq [K]$, minimum dominating set $\Sset$, OCO algorithm $\mathcal{A}$ with domain $\doma{W} \subseteq \reals^d$, $\gamma \geq 0$, and gap map $a:\reals^{K \times d} \times \reals^d \rightarrow [0, 1]$
\For{$t = 1 \ldots T$}
\State Obtain $\x_t$ 
\State Let $\ystart = \argmax_{k} \inner{\W_t^k}{\x_t}$ \Comment{max-margin prediction}
\If{$\ystart \in \Qset$}
\State Set $\gamma_t = 0$ 
\Else
\State Set $\gamma_t = \min\left\{\half, \gamma\Big/ \sqrt{\big|\{s \leq t: y^\star_s \not \in \Qset\}\big|}\right\}$ \Comment{exploration rate}
\EndIf
\State Set $\zeta_t = \id[\gamma_t \leq a(\W_t, \x_t)]$ 
\State $\p_t' = \big(1 - \zeta_t a(\W_t, \x_t) - (1 - \zeta_t) \gamma_t\big) \e_{\ystart} + \zeta_t a(\W_t, \x_t) \frac{1}{K}\1 + (1 - \zeta_t) \frac{\gamma_t}{\rho}\1_{S}$ \label{line:sampleprob}
\State Predict with label $y_t' \sim \p_t'$
\State Compute ${\dt v_t = \frac{\id[y_t \in \textnormal{out}(y_t')]}{P_t(y_t \in \textnormal{out}(y_t'))} }$ \Comment{$y_t$ is observed only when $y_t \in \textnormal{out}(y_t')$}
\State Set $\lhat_t(\W_t) = v_t \ell_t(\W_t)$ \Comment{compute loss estimates} \label{line:lhat}
\State Send $\lhat_t$ to $\mathcal{A}$ and get $\W_{t+1}$ in return 
\EndFor
\end{algorithmic}
\end{algorithm}

In this section we introduce \textproc{Gappletron}, whose pseudocode is presented in Algorithm~\ref{alg:gappletron}. As input, the algorithm takes information about the graph $\Gset$ in the form of a minimum dominating set $\Sset$ and a (possibly empty) set of revealing actions $\Qset$. \textproc{Gappletron} maintains a parameter $\W_t \in \doma{W} \subseteq \reals^{d \times K}$ and uses some full information Online Convex Optimization (OCO) algorithm $\Aset$ to update the vector form of $\W_t$. Our results hold whenever $\Aset$ satisfies the condition that $\sumT \big(\lhat_t(\W_t) - \lhat_t(\U)\big)$ be at most of order $h(\U)\sqrt{\sumT \|\nabla \lhat_t(\W_t)\|^2}$, where $\lhat_t$ are the estimated losses computed at line~\ref{line:lhat} of Algorithm~\ref{alg:gappletron} and
$h : \doma{W} \to \reals_+$ is any upper bound on the norm of $\U\in\doma{W}$. Since practically any OCO algorithm can be tuned to have such a guarantee --- see \citep{orabona2018scale} --- this is a mild requirement. Whereas \textproc{Gaptron} %
is only able to use Online Gradient Descent (OGD) with a fixed learning rate,  \textproc{Gappletron} allows for more flexibility, which in turn may lead to different guarantees on the surrogate regret.
For example, if the learner runs an OCO algorithm with a good dynamic regret bound \citep{zinkevich2003}, then \textproc{Gappletron} enjoys a good dynamic surrogate regret bound. Furthermore, the guarantee of $\Aset$ allows us to derive stronger results in the separable setting while maintaining a similar guarantee as \textproc{Gaptron} %
in the non-separable setting, which is not possible when using OGD with a fixed learning rate. Additional inputs to  \textproc{Gappletron} are $\gamma > 0$, which is used to control the exploration rate of the algorithm in the partial information setting, and the gap map $a$, whose role we explain below.

The predictions of Algorithm~\ref{alg:gappletron} are sampled from $\p_t'$ defined in line~\ref{line:sampleprob},
where $\e_{\ystart}$ is the basis vector in the direction of the margin-based linear prediction 
$\ystart = \argmax_{k} \inner{\W_t^k}{\x_t}$. The gap map $a:\reals^{K \times d} \times \reals^d \to [0, 1]$ controls the mixture between $\e_{\ystart}$ and the uniform exploration term $\frac{1}{K}\1$. For brevity, we sometimes write $a_t$ instead of $a(\W_t, \x_t)$.  
The single most important property of \textproc{Gappletron} is presented in the following Lemma.
\begin{lemma}
\label{lem: upper bound zo loss}
Fix any feedback graph $\Gset$ and suppose that, for all $t$, $\ell_t$ is a regular surrogate loss with respect to $\ell$. Then  \textproc{Gappletron}, run on $\Gset$ with %
$a$ such that $a(\W_t, \x_t) = \ell(\W_t, \x_t, \ystart)$, satisfies
\begin{align*}
    \zopt \leq \frac{K-1}{K} \ell_t(\W_t) + \gamma_t.
\end{align*}
\end{lemma}
\begin{proof}
First, observe that $\zopt \leq (1 - a_t)\id[y_t^\star \not = y_t] + a_t \frac{K-1}{K} + \gamma_t$, 
since $\zeta$, $(1 - \zeta)$, and the cost of exploration are at most $1$.
To conclude the proof we claim that the first two terms in the right-hand side are upper bounded by $\frac{K-1}{K} \ell_t(\W_t)$. We show that by considering two cases. In the first case $\ystart = y_t$ and the inequality simply follows by substituting
$a_t = \ell\big(\W_t, \x_t, \ystart\big) = \ell_t(\W_t)$.
In the second case $\ystart \not = y_t$ and we have that
\begin{equation*}
\begin{split}
    &(1 - a_t)\id[\ystart \not = y_t] + a_t\frac{K-1}{K} = 1 - \frac{1}K\ell\big(\W_t, \x_t, \ystart\big) \\
    & = 1 - \frac{1}{K} \ell\big(\W_t, \x_t, \ystart\big) - \frac{K-1}{K}\ell\big(\W_t, \x_t, y_t\big) + \frac{K-1}{K}\ell_t(\W_t)  \leq \frac{K-1}{K}\ell_t(\W_t),
\end{split}
\end{equation*}
where the inequality is due to equation \eqref{eq:wrong plus right condition} in the definition of regular surrogate losses.
\end{proof}
Although the \textproc{Gaptron} algorithm uses similar predictions, it is not clear how to choose $a$ such that a property similar to the one described in Lemma \ref{lem: upper bound zo loss} holds. Rather, \citet{vanderHoeven2020exploiting} derives a different gap map for the hinge loss, the smooth hinge loss, and the logistic loss, and analyses the surrogate regret separately for each loss. With Lemma \ref{lem: upper bound zo loss} in hand, we simplify the analysis and --- at the same time --- also generalize the results of \citet{vanderHoeven2020exploiting} to other surrogate losses. Furthermore, Lemma \ref{lem: upper bound zo loss} also allows us derive surrogate regret bounds that hold with high probability.

What Lemma \ref{lem: upper bound zo loss} states is that with regular surrogate losses and $a(\W_t, \x_t) = \ell\big(\W_t, \x_t, \ystart\big)$ the expected zero-one loss of  \textproc{Gappletron} can be upper bounded by $\frac{K-1}{K} \ell_t(\W_t)$ plus the cost of exploration. While at first this may seem of little interest, note that we want to bound the zero-one loss in terms of $\ell_t$ rather than $\frac{K-1}{K} \ell_t$. Compared to standard algorithms, this gains us a $-\frac{1}{K}\ell_t(\W_t)$ term in \emph{each} round, which we can use to derive our results. To see how,
observe that  \textproc{Gappletron} uses an OCO algorithm $\Aset$ to update $\myvec(\W_t)$ on each round. Suppose that, for some $h:\doma{W} \to \reals$ and $\U \in \doma{W}$, Algorithm $\Aset$ satisfies 
\begin{align}\label{eq:algA requirement}
    \sumT\left(\lhat_t(\W_t)-\lhat_t(\U)\right) \leq h(\U)\sqrt{\sumT \|\ghat_t\|^2},
\end{align}
where $\ghat_t = v_t \nabla \ell_t(\W_t)$. For simplicity, now assume we are in the full information setting (e.g., $v_t = 1$ for all $t$). Since $\ell_t$ is a regular surrogate loss, we can use
$
    \|\nabla \ell_t(\W)\|^2 \leq 2 L\,\ell_t(\W)
$
and $\sqrt{ab} = \half \inf_{\eta > 0} \left\{{a}/{\eta} + \eta b\right\}$ to show that 
\begin{align*}
    h(\U)\sqrt{\sumT \|\ghat_t\|^2} - \sumT\frac{1}{K}\ell_t(\W_t) \leq  h(\U)\sqrt{\sumT 2L \ell_t(\W_t)} - \sumT\frac{1}{K}\ell_t(\W_t) \leq \frac{KLh(\U)^2}{2}.
\end{align*}
This means that in the full information setting
the surrogate regret of \textproc{Gappletron}
is independent of the number of rounds. In the partial information setting some additional steps are required, but the idea remains essentially the same.
We formalize the aforementioned ideas in the following Lemma, whose proof is deferred to Appendix \ref{app:gappletron}.
\begin{restatable}{relemma}{lemsurrogategap}
\label{lem: surrogate gap}
Fix any feedback graph $\Gset$ and suppose that, for all $t$, $\ell_t$ is a regular surrogate loss with respect to $\ell$. If $\doma{A}$ satisfies equation~\eqref{eq:algA requirement} then, for any realization of the randomized predictions $y_1',\ldots,y_T'$,  \textproc{Gappletron}, run on $\Gset$ with gap map $a$ such that $a(\W_t, \x_t) = \ell(\W_t, \x_t, \ystart)$, satisfies
\begin{align*}
    \sumT &\zopt
\le
    \sumT \lhatt(\U) + \sumT \gamma_t
\\&
    + \inf_{\eta > 0}  \bigg\{ \frac{h(\U)^2}{2\eta} + \sumT \left(\frac{K-1}{K}\ell_t(\W_t) - \vt \ell_t(\W_t) + \eta \vt^2 L \ell_t(\W_t) \right)\bigg\} \qquad \forall\,\U \in \doma{W}~.
\end{align*}
\end{restatable}

\section{Bounds that Hold in Expectation}\label{sec:expectation}
In this section we present bounds on the surrogate regret that hold in expectation. For brevity we use $\Mset_T = \sumT \zot$. 
We now state a simplified version of Theorem~\ref{th:expectation bound}, whose full statement and proof can be found in Appendix \ref{app:expectation}.
\begin{theorem}\label{th: informal expectation bound}
Let $\Gset$ be any feedback graph with dominating number $\rho$ and revealing action set $\Qset$. Suppose that, for all $t$, $\ell_t$ is a regular surrogate loss with respect to $\ell$. If $\doma{A}$ satisfies equation~\eqref{eq:algA requirement} then  \textproc{Gappletron}, run on $\Gset$ and $\doma{A}$ with gap map $a$ such that $a(\W_t, \x_t) = \ell(\W_t, \x_t, \ystart)$, satisfies
\begin{align*}
    & \Eb{\regret_T} = O\left(\Eb{\max\bigg\{\frac{K^2 L h(\U)^2}{\max\{1, |\Qset|\}},  h(\U)\sqrt{\rho K L\revt}}\bigg\} \right) \qquad \forall\,\U \in \doma{W}~.
\end{align*}
Furthermore, for all $\U \in \doma{W}$ such that $\sumT \ell_t(\U) = 0$,  \textproc{Gappletron} satisfies:
\begin{align*}
    & \Eb{\Mset_T} =  O\left( \Eb{\max\left\{h(\U)\sqrt{\rho L\revt}, \frac{K L h(\U)^2}{\max\{1, |\Qset|\}}\right\}} - \frac{1}{K}\Eb{\sumT \ell_t(\W_t)}\right).
\end{align*}
\end{theorem}
In the full information setting we clearly have that $\Qset = [K]$. Hence, using OGD as $\doma{A}$ with an appropriated learning rate, the second statement in Theorem \ref{th: informal expectation bound} reduces to $\E\big[\Mset_T\big] \le 4 L\|\U\|_2^2 - \sumT \frac{1}{K} \ell_t(\W_t),$ which improves the results of \textproc{Gaptron} in the separable case by at least a factor $K$. Interestingly, compared to standard bounds for the separable case, such as the \textproc{Perceptron} bound, there is a negative term which seems to further lower the cost of learning how to separate the data.
Similarly, in the partial information setting, the bound for the separable case in Theorem \ref{th: informal expectation bound} has a reduced dependency on $K$ compared to the non-separable case, obtaining similar improvements over \textproc{Gaptron} as in the full information setting. 

For the non-separable case, Theorem \ref{th: informal expectation bound} generalizes \textproc{Gaptron} in two directions. The most prominent direction is the extension is to feedback graphs,
where our analysis reveals a surprising phenomenon: 
Theorem \ref{th: informal expectation bound} in fact shows that the surrogate regret in the label efficient setting (and in any setting where $\rho < K$) is actually smaller than in the bandit setting, where $\rho=K$.
Intuitively, this is due to the fact that our algorithm only updates when $y_t$ is known. In the bandit setting, we need to explore all labels to find $y_t$, while in label efficient classification we can just play whichever action is the revealing action, and find $y_t$. This implies that exploration in label efficient classification is easier than in the bandit setting. Note that in the bandit setting, playing $y_t' \not = y_t$ also provides the learner with information. Perhaps by using this information effectively, one is able to improve our surrogate regret bounds, but as of yet it is not clear how to use knowledge of the wrong label. The second extension is that the bounds in Theorem \ref{th: informal expectation bound} hold for all regular surrogate loss functions with the same gap map defined by the surrogate loss, rather than only for a limited number of loss functions and ad-hoc gap maps as it was the case with \textproc{Gaptron}.

\section{Bounds that Hold with High Probability}\label{sec:high probability}
We 
now present bounds on the surrogate regret that hold with high probability. After proving a general surrogate regret bound, we derive a corresponding bound, with improved guarantees, for the full information setting. The bound for the partial information setting can be found in Theorem~\ref{th:high probability} in Appendix \ref{app:high prob}, which implies Theorem~\ref{th:informal high probability} below.
Let the maximum loss over all rounds be $\lmax = \max_{t, \W \in \doma{W}} \ell_t(\W)$. 
\begin{theorem}\label{th:informal high probability}
With probability at least $1 - \delta$,  \textproc{Gappletron} satisfies: 
\begin{align*}
    \regret_T = \Ob{\sqrt{\left(Lh(\U)^2 + \lmax \ln(1/\delta)\right)K\rho T}} \qquad \forall\,\U\in\doma{W}
\end{align*}
Furthermore, for all $\U \in \doma{W}$ such that $\sumT \ell_t(\U) = 0$, with probability at least $1-\delta$  \textproc{Gappletron} satisfies:
\begin{align*}
    \Mset_T =  \Ob{\sqrt{ \left(L h(\U)^2 + K \lmax \ln(1/\delta)\right) \rho T}}.
\end{align*}
\end{theorem}
Theorem \ref{th:informal high probability} shows that Algorithm \ref{alg:gappletron} has $O(h(\U)\sqrt{\rho KT})$ surrogate regret in the worst case, with high probability. As far as the authors are aware, this is the first high-probability surrogate regret bound for a margin-based classifier in the partial information setting. Similarly to the bounds in expectation, the worst-case surrogate regret is the largest in the bandit setting ($\rho = K$) and the smallest in label efficient classification ($\rho = 1$). Unlike the bounds in expectation, where the surrogate regret was at least a factor $\sqrt{K}$ smaller in the separable case, the improvement in Theorem \ref{th:informal high probability} is less apparent, but the surrogate regret still has a better dependence on $K$ in the separable case. In particular, all the terms with $h(\U)$ have a better dependence on $K$. 

In the full information setting the dependence on $\lmax$ can be removed. This cannot be achieved in the partial information setting, due to the necessity of estimating the surrogate loss.
If $\doma{W}$ has a bounded radius $B$ and $\ell_t$ has gradients bounded by $G$, then $\lmax \leq 1 + BG$ by convexity. The bound for the full information setting can be found in Theorem \ref{th:fullinfo hp}. In the separable case of the full information setting, the bound does not depend on $K$, which is not the case for Theorem \ref{th:informal high probability} due to the need to control the surrogate loss estimates.
\begin{restatable}{retheorem}{thfullinfohp}
\label{th:fullinfo hp}
Under the conditions of Lemma \ref{lem: surrogate gap}, with probability at least $1 - \delta$,  \textproc{Gappletron} satisfies
\begin{align*}
    \Mset_T \leq \sumT \ell_t(\U) + K L h(\U)^2 + \frac{3K+1}{2}\ln\frac{1}{\delta} \qquad \forall\,\U\in\doma{W}~.
\end{align*}
Furthermore, for all $\U \in \doma{W}$ such that $\sumT \ell_t(\U) = 0$, then with probability at least $1 - \delta$,  \textproc{Gappletron} satisfies
$
    \Mset_T \leq 4 L h(\U)^2 + \frac{11}{4}\ln\frac{1}{\delta} 
$.
\end{restatable}
We provide a proof sketch of the full information versions of Theorem~\ref{th:fullinfo hp}.
The proof for the partial information setting is essentially the same, with some extra steps to control the estimates of the surrogate losses. Let $z_t = \left(\zot - \zopt\right)$. The proof
relies on \citep[Theorem 1]{beygelzimer2011contextual} --- see Lemma \ref{lem: freedman} in this paper, which, when translated to our setting, states that with probability at least $1-\delta$, 
\begin{equation*}
    \sumT z_t \leq \sqrt{3 \ln\frac{1}{\delta} \sumT \E_t\left[z_t^2\right]} + 2\ln\frac{1}{\delta}~.
\end{equation*}
Since the variance is bounded by the second moment,
$
    \E_t\left[z_t^2\right] \leq \E_t\left[\zot\right] \leq \frac{K-1}{K} \ell_t(\W_t)
$,
where the last inequality is due to Lemma \ref{lem: upper bound zo loss}. By using $\sqrt{ab} = \inf_{\eta > 0} \frac{a}{2\eta} + \frac{\eta}{2}b$ and applying Lemma~\ref{lem: surrogate gap}, we find that
\begin{equation*}
    \Mset_T \leq \sumT \ell_t(\U) + \inf_{\eta \in (0, 1]} \left\{\frac{h(\U) - 3\ln\delta}{2\eta} + \sumT \left(\eta \left(L + \frac{K-1}{K}\right) - \frac{1}{K}\right) \ell_t(\W_t)\right\}, 
\end{equation*}
with probability at least $1 - \delta$. After choosing an appropriate $\eta$, this gives us a $O\big(K h(\U)^2\big)$ surrogate regret bound with high probability.

\section{Lower Bounds}\label{sec:lower bounds}
Corollary \ref{cor:lower} below here shows that the bound of Theorem \ref{th: informal expectation bound} cannot be significantly improved.
\begin{corollary}
\label{cor:lower}
Let $A$ be a possibly randomized algorithm for the online multiclass classification setting with feedback graphs. Then, for any $B = \Omega(1)$, the surrogate regret of $A$ with respect to the smooth hinge loss must satisfy
\[
    \Eb{\Mset_T}
=
    \min_{\U\in\doma{W}}\sumT \ell_t(\U) + \Omega\big(KB^2 + \sqrt{T}\big)
\]
where $K$ is the number of classes, the feature vectors $\x_t$ satisfy $\norm{\x_t}_2 = \Theta(1)$ for all $t$, and $\doma{W} = \theset{\W}{\norm{\W} \le B}$.
\end{corollary}
Corollary \ref{cor:lower} is implied by Theorems~\ref{th:apple lower bound} and~\ref{th:full info lower bound} in Appendix \ref{app:lower bounds}. The proof of Theorem~\ref{th:apple lower bound} builds on the lower bound of \citet{daniely2015strongly} for strongly-adaptive regret. The feedback graph considered in the proof is filtering with two classes: a blind class (no outgoing edges) and a revealing class. In the proof, we show that the algorithm either explores too much, in which case the lower bound trivially holds, or the algorithm explores too little, in which case the environment can trick the algorithm into playing the wrong action by exploiting the blind class.  

\section{Experiments}\label{sec:experiments}

\begin{figure}
    \centering
    \includegraphics[width = \textwidth]{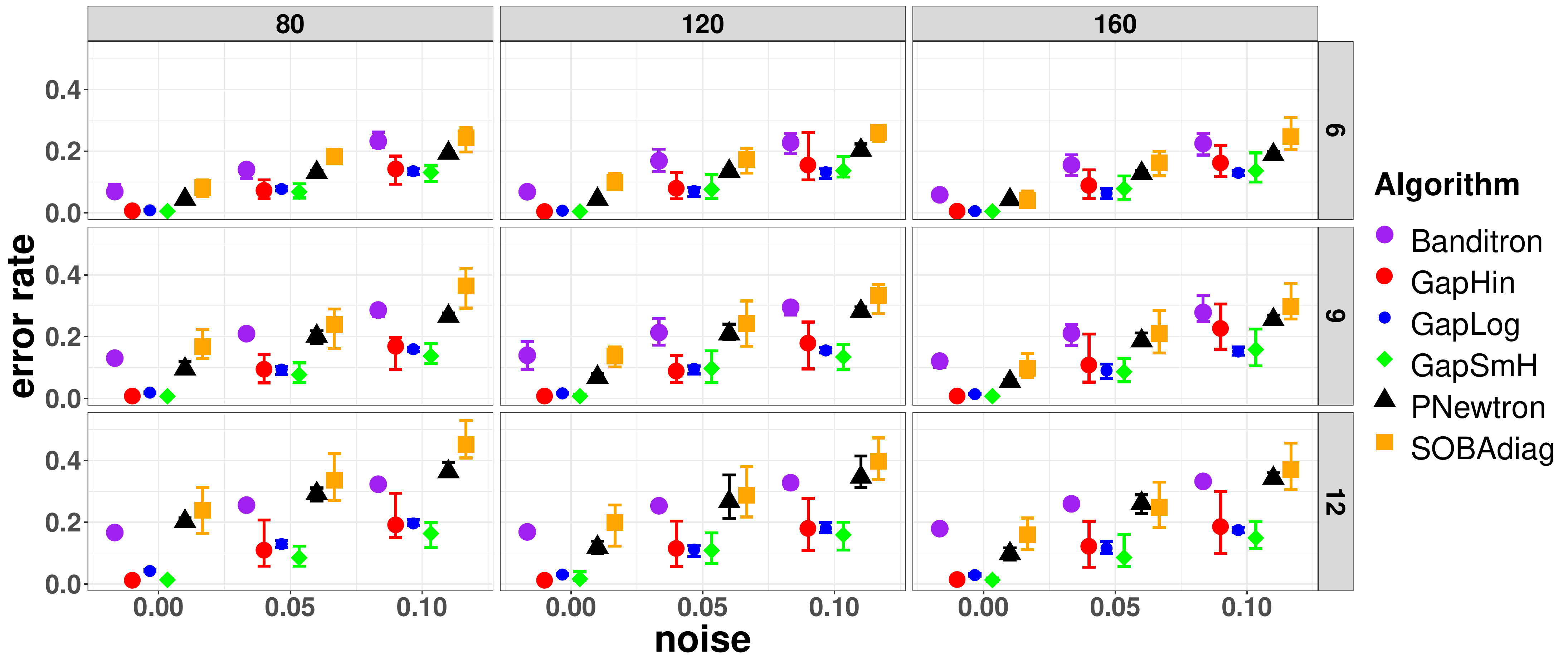}
    \caption{Results of the synthetic experiments for the bandit setting. The plot shows the best results of algorithms with parameters suggested by theory, or tuned with all parameters set to 1, except for $T$. The rows indicate different values for $K$ and the columns different values for $d$. Whiskers show the minimum and the maximum error rate over ten repetitions.}
    \label{fig:Synbestbandit}
\end{figure}
We empirically evaluated the performance of \textproc{Gappletron} on synthetic data in the bandit, multiclass filtering, and full information settings. Similarly to the SynSep and NonSynSyp datasets described in \citep{kakade2008efficient}, we generated synthetic datasets with $d \in \{80, 120, 160\}$, $K \in \{6, 9, 12\}$, and the label noise rate in $\{0, 0.05, 0.1\}$. Due to space constraints, we only report part of the experiments for the bandit setting in the main text, see Figure \ref{fig:Synbestbandit}. In the bandit setting we used worst case tuning for the algorithms with the parameters suggested by theory, or set all parameters to 1, except for $T$. Initially we only used theoretical tuning for all algorithms, but we found that two algorithms we compared with did not have satisfactory results. A more detailed description of the results, including how we generated data and tuned the algorithms, can be found in Appendix \ref{app:experiments}. 

In the bandit setting, we compared  \textproc{Gappletron} with the following baselines: PNewtron (the diagonal version of Newtron, by \citet{hazan2011newtron}), SOBAdiag (the diagonal version of SOBA, by \citet{beygelzimer2017efficient}), and the importance weighted version of Banditron\footnote{This is a version different from the one described by \citet{kakade2008efficient}, in particular, we replaced $\widetilde{U}^t$ in their Algorithm~1 with the gradient of the importance weighted hinge loss.} \citep{kakade2008efficient}. We opted to use the diagonal versions of Newtron and SOBA for computational reasons. We chose the importance weighted version of Banditron because the standard version did not produce satisfactory results. We used three surrogate losses for  \textproc{Gappletron}: the logistic loss $\ell_t(\W_t) = -\log_K q(\W_t, \x_t, y_t)$ where $q$ is the softmax, the hinge loss defined in \eqref{eq:multiclass hinge}, and the smooth hinge loss \citep{rennie2005loss}, denoted by GapLog, GapHin, and GapSmH respectively. The OCO algorithm used with all losses is Online Gradient Descent, with learning rate $\eta_t =\left(10^{-8} + \sumt \|\nabla \lhat_j(\W_t)\|_2^2\right)^{-1/2}$ and no projections. 

As shown in Figure \ref{fig:Synbestbandit}, on average all versions of \textproc{Gappletron} outperform the baselines in the bandit setting. GapHin appears to be more unstable than the other versions of \textproc{Gappletron}. We suspect this is due to the fact that GapHin explores less than its counterparts. In multiclass spam filtering (Figure \ref{fig:SynbestRA} in Appendix \ref{app:experiments}), we see that GapLog makes more mistakes than its counterparts for $K > 6$. We suspect this is due to the fact that with logistic loss, the gap map is never zero, which implies that GapLog picks an outcome uniformly at random more often than GapHin and GapSmH, while not gaining any information. Due to this behaviour GapLog makes more mistakes than necessary, which we also observe in the full information setting. In the bandit setting, the additional exploration leads to additional stability for GapLog, as indicated by the small range of performance of GapLog. In all cases, increasing the exploration rate increased the stability of \textproc{Gappletron}, which is very much in agreement with Theorem~\ref{th:high probability}.

\section{Future work}
There are several intriguing research directions left open to pursue. While our lower bound holds for general feedback graphs, it is not clear whether our bounds are tight for the bandit setting. Either providing a lower bound or an improved algorithm for the bandit setting remains thus open. Our results show that it is possible to obtain improved bounds for the separable case while maintaining satisfactory results for the non-separable case. However, as \citet{beygelzimer2019bandit} show, it is possible to obtain even better guarantees in the separable case of the bandit setting. An algorithm guaranteeing $O(K\|\U\|^2)$ mistakes in the separable case and $O(K\sqrt{T})$ surrogate regret in the non-separable case, without prior knowledge of the separability, would therefore be an interesting contribution. 

\paragraph{Acknowledgements} Van der Hoeven and Cesa-Bianchi are supported by the MIUR PRIN grant Algorithms, Games, and Digital Markets (ALGADIMAR) and by the EU Horizon 2020 ICT-48 research and innovation action under grant agreement 951847, project ELISE (European Learning and Intelligent Systems Excellence). Fusco is supported by the ERC Advanced Grant 788893 AMDROMA “Algorithmic and Mechanism Design
Research in Online Markets” and the MIUR PRIN project ALGADIMAR “Algorithms, Games, and Digital Markets".

\bibliography{myBib}



\appendix

\section{Hinge loss}\label{app:hinge loss}
The multiclass hinge loss is defined by
\begin{equation}\label{eq:multiclass hinge}
    \ell_t(\W) = 
    \begin{cases}
        \max\{1 - m_t(\W, y_t), 0\} & \text{ if } \mstart < \kappa  \\
        \max\{1 - m_t(\W, y_t), 0\} & \text{ if } \ystart \not = y_t \text{ and } \mstart \geq \kappa  \\
        0 & \text{ if } \ystart = y_t \text{ and } \mstart \geq \kappa,
    \end{cases}
\end{equation}
where $m_t(\W, y) = \inner{\W^y}{\x_t} - \max_{k \not= y} \inner{\W^k}{\x_t}$, $\mstart = \max_k m_t(\W_t, k)$, and $\kappa \in [0, 1]$. Setting $\kappa = 0$ yields the multiclass hinge loss used in common implementations of the Perceptron. An alternative version of Lemma \ref{lem: upper bound zo loss} which holds for the hinge loss can be found in Lemma \ref{lem: hinge upper bound zo}. This can then be used to derive similar results for the hinge loss as for regular surrogate losses. 
\begin{lemma}\label{lem: hinge upper bound zo}
Let $\ell_t$ be the multiclass hinge loss with $\kappa = \half$ and let $a(\W_t, \x_t) = \ell\big(\W_t, \x_t, \ystart\big)$, where $\ell(\cdot,\x_t,y_t) = \ell_t$. 
Then \textproc{Gappletron} satisfies
\begin{align*}
    \zopt \leq \max\left\{\frac{2}{3}, \frac{K-1}{K}\right\} \ell_t(\W_t) + \gamma_t.
\end{align*}
Furthermore, $\ell_t$ satisfies $\|\nabla \ell_t(\W_t)\|^2 \leq 4 \|\x_t\|^2 \ell_t(\W_t)$.
\end{lemma}
\begin{proof}[Proof of Lemma \ref{lem: hinge upper bound zo}]
First, we have that
\begin{align*}
    \zopt & \le \big(1 - \zeta_t a_t - (1-\zeta_t) \gamma_t\big)\id[y_t^\star \not = y_t] + \zeta_t a_t \frac{K-1}{K} + (1 - \zeta_t) \gamma_t \\
    & \le (1 - a_t)\id[y_t^\star \not = y_t] + a_t \frac{K-1}{K} + \gamma_t,
\end{align*}
where we used that $\zeta \in \{0,1\}$ and the fact the number of mistakes while uniformly exploring on the dominating set is upper bounded by 1.

To conclude the proof of the first statement we argue that the first two summands of the right hand side are upper bounded by $\frac{K-1}{K} \ell_t(\W_t)$. In order to show that, we split the analysis into two cases. In the first case $\ystart = y_t$ and the inequality follows by simply substituting $a_t = \ell\big(\W_t, \x_t, \ystart\big) = \ell_t(\W_t)$. In the second case $\ystart \not = y_t$ and we have that
\begin{equation*}
\begin{split}
    \mstart + m_t(\W_t, y_t) = & \inner{\W_t^{\ystart}}{\x_t} - \max_{k \not= \ystart}\inner{\W_t^k}{\x_t} +  \inner{\W_t^{y_t}}{\x_t} - \max_{k \not= y_t} \inner{\W_t^k}{\x_t} \\
    = & \inner{\W_t^{y_t}}{\x_t} - \max_{k \not= \ystart}\inner{\W_t^k}{\x_t}\\
    \leq & \inner{\W_t^{y_t}}{\x_t} - \inner{\W_t^{y_t}}{\x_t} = 0
\end{split}
\end{equation*}
and thus
\begin{align}\label{eq:mtstar + mt}
    \mstart \leq -m_t(\W_t, y_t).
\end{align}
Since $\ystart \not = y_t$ we also have that
\begin{equation}\label{eq:kappamstart}
\begin{split}
    & (1 - a_t)\id[\ystart \not = y_t] + a_t\frac{K-1}{K}  \\
    & = \Big(1 - \ell\big(\W_t, \x_t, \ystart\big)\Big) + \ell\big(\W_t, \x_t, \ystart\big) \frac{K-1}{K}  \\
    & =  1 - \frac{1}{K} \ell\big(\W_t, \x_t, \ystart\big)  \\
    & =  1 - \frac{1}{K}\id[\mstart < \kappa]\big(1 - m_t(\W_t, \ystart)\big).
\end{split}
\end{equation}
Now, if $\mstart < \kappa$ then by equations \eqref{eq:mtstar + mt} and \eqref{eq:kappamstart} we have
\begin{align*}
    & (1 - a_t)\id[\ystart \not = y_t] + a_t\frac{K-1}{K} = \frac{K-1}{K} + \frac{1}{K} \mstart \leq \frac{K-1}{K}\left(1 + \mstart \right) \leq \frac{K-1}{K} \ell_t(\W_t).
\end{align*}
If $\mstart \geq \kappa$, $a_t = 0$. Therefore, by equations \eqref{eq:mtstar + mt} and \eqref{eq:kappamstart} we have that
\begin{align*}
    (1 - a_t)\id[\ystart \not = y_t] + a_t\frac{K-1}{K} = \frac{1 + \mstart}{1 + \mstart} 
    \leq \frac{1 - m_t(\W_t, y_t)}{1 + \kappa} = \frac{1}{1 + \kappa} \ell_t(\W_t).
\end{align*}
Setting $\kappa = \half$ completes the proof of the first statement. 

For the proof of the second statement, first assume that $\ystart = y_t$. The case where $\mstart \geq \kappa$ is straightforward, so suppose that $\mstart < \kappa$, in which case we have that 
\begin{align*}
    \|\nabla \ell_t(\W_t)\|^2 \leq & 2 \|\x_t\|^2 \\
    = & \frac{1-\mstart}{1-\mstart} \|\x_t\|^2 \\
    \leq & 4\|\x_t\|^2 \ell_t(\W_t).
\end{align*}
The case where $\ystart \not = y_t$ is evident after observing that $\ell_t(\W_t) \geq 1$ in that case. 
\end{proof}

\section{Details of Section \ref{sec:gappletron} (Gappletron)}
\label{app:gappletron}

\lemsurrogategap*

\begin{proof}[Proof of Lemma \ref{lem: surrogate gap}]
By adding and subtracting the surrogate loss of the learner and using the guarantee of $\doma{A}$ we have
\begin{equation*}
\begin{split}
    & \sumT  \left(\zopt - \lhat_t(\U)\right) \\
    & =  \sumT \left(\zopt - \lhat_t(\W_t)\right) + \sumT \left(\lhat_t(\W_t) - \lhat_t(\U)\right) \\
    & \leq  \sumT \left(\zopt - \lhat_t(\W_t)\right) + h(\U)\sqrt{\sumT \|\ghat_t\|^2} \\
    & \leq \inf_{\eta > 0} \left\{\frac{h(\U)^2}{2\eta} + \sumT \left( \zopt - \lhat_t(\W_t) + \frac{\eta}{2}\|\ghat_t\|^2\right)\right\},
\end{split}
\end{equation*}  
where in the last inequality we used $\sqrt{ab} = \inf_{\eta > 0} \left\{\frac{a}{2\eta} + \frac{\eta}{2} b\right\}$. Recalling that $\ghat_t = v_t \nabla \ell_t(\W_t)$, we continue by applying Lemma \ref{lem: upper bound zo loss}:
\begin{equation}\label{eq:surrogate gap proof}
\begin{split}
    & \sumT  \left(\zopt - \lhat_t(\U)\right) \\
    & \leq  \inf_{\eta > 0} \left\{\frac{h(\U)^2}{2\eta} + \sumT \left(\frac{K-1}{K}\ell_t(\W_t) + \gamma_t - \lhat_t(\W_t) + \frac{\eta}{2}\|\ghat_t\|^2\right) \right\} \\
    & = \inf_{\eta > 0} \left\{\frac{h(\U)^2}{2\eta} + \sumT \left(\frac{K-1}{K}\ell_t(\W_t) + \gamma_t - \vt \ell_t(\W_t) + \frac{\eta \vt^2}{2}\|\nabla \ell_t(\W_t)\|^2 \right)\right\} \\
    & \leq \inf_{\eta > 0} \left\{\frac{h(\U)^2}{2\eta} + \sumT \left(\frac{K-1}{K}\ell_t(\W_t) + \gamma_t - \vt \ell_t(\W_t) + \eta \vt^2 L \ell_t(\W_t) \right)\right\},
\end{split}
\end{equation}
where in the final inequality we used equation \eqref{eq:gradient condition}. The lemma's statement follows from rearranging the last inequality. %
\end{proof}

\section{Details of Section \ref{sec:expectation} (Bounds that hold in expectation)}\label{app:expectation}

\begin{theorem}\label{th:expectation bound}
Under the conditions of Lemma \ref{lem: surrogate gap},  \textproc{Gappletron} with $\gamma = \half  h(\U)\sqrt{K\rho L}$ satisfies:
\begin{align*}
    & \E\left[\sumT \id[y_t' \not = y_t]\right] \\
    & \le \E\left[\sumT \ell_t(\U)\right] + \max\bigg\{\frac{2 K^2Lh(\U)^2}{\max\{1, |\Qset|\}}, 2\Eb{ h(\U)\sqrt{\rho K L\revt}}\bigg\} 
\end{align*}
Furthermore, if there exists a $\U \in \doma{W}$ such that $\sumT \ell_t(\U) = 0$ for all realizations of the learners' actions,  \textproc{Gappletron} with $\gamma = h(\U) \sqrt{L \rho}$ satisfies:
\begin{align*}
    & \Eb{\sumT \zot} \\
    & \leq  \Eb{\max\left\{ 4  h(\U)\sqrt{\rho L\revt}, \frac{4 K L h(\U)^2}{\max\{1, |\Qset|\}}\right\}} - \frac{1}{K}\Eb{\sumT \ell_t(\W_t)}.
\end{align*}
\end{theorem}
\begin{proof}[Proof of Theorem \ref{th:expectation bound}]
Denote by $\vmax = \max\{1, \max_t v_t\}$. Observe that $\E_t[v_t] = 1$ and $\E_t[v_t^2] \leq \E_t[\vmax]$. We start by applying Lemma \ref{lem: surrogate gap} and taking expectations:
\begin{align*}
    & \E\left[\sumT \id[y_t' \not = y_t]\right] \\
    &-  \E\left[\sumT \ell_t(\U)\right] - \Eb{\sumT \gamma_t} \\
    &\le \Eb{\inf_{\eta > 0}  \bigg\{ \frac{h(\U)^2}{2\eta} + \sumT \left(\frac{K-1}{K}\ell_t(\W_t) - \vt \ell_t(\W_t) + \eta \vt^2 L \ell_t(\W_t) \right)\bigg\}} \\
    &\le  \inf_{\eta > 0}  \bigg\{ \Eb{\frac{h(\U)^2}{2\eta}} + \Eb{\sumT \left(\eta \vmax L - \frac{1}{K}\right) \ell_t(\W_t) }\bigg\} \\
    &\leq \Eb{\frac{\vmax K L h(\U)^2}{2}},
\end{align*}
where the last inequality follows from setting $\eta = \frac{1}{K L \vmax}$. By using $\sum_{j = 1}^J \frac{1}{\sqrt{j}} \leq 2\sqrt{J}$ we can see that $\sumT \gamma_t \leq 2\gamma\sqrt{\revt}$. Now, observe that if $\ystart \in \Qset$ then $P_t(y_t \in \textnormal{out}(y_t')) \geq \frac{|\Qset|}{K}$ and if $\ystart \not \in \Qset$ then $P_t(y_t \in \textnormal{out}(y_t')) \geq \min \left\{\frac{1}{2\rho}, \frac{\gamma_t}{\rho}\right\} \geq \min \left\{\frac{1}{2K}, \frac{\gamma_T}{\rho}\right\}$. This means that
\begin{equation}\label{eq:vmax}
    \vmax \leq \max\left\{\frac{\rho}{\gamma_T}, \frac{K}{\max\{1, |\Qset|\}}\right\}.
\end{equation} 
Recall that $\gamma_T = \min\{\half, {\gamma}/{\sqrt{\revt}}\}$. If $\rho > 1$ then $|\Qset| = 0$ which means that if $\half < \frac{\gamma}{\sqrt{T}}$ then $\vmax \leq 2K$. On the other hand, if $\rho = 1$ then $|\Qset| \geq 1$ which means that if $\half < \frac{\gamma}{\sqrt{T}}$ then $\vmax \leq \frac{2K}{|\Qset|}$. This in turn means that 
\begin{equation}\label{eq:vmaxE}
    \vmax \leq \max\left\{\frac{\rho \sqrt{\revt}}{\gamma}, \frac{2K}{\max\{1, |\Qset|\}}\right\}.
\end{equation}
Rearranging the previous inequality and substituting in $\gamma= h(\U)\sqrt{L|\mathcal{S}|}$, 
\begin{align*}
    \E\left[\sumT \id[y_t' \not = y_t]\right] & \le \E\left[\sumT \ell_t(\U)\right] + \Eb{h(\U)\sqrt{\rho K L \revt}} \\
    & + \Eb{\max\left\{ h(\U) \sqrt{\rho K L \revt}, \frac{2K^2 L h(\U)^2}{2\max\{1, |\Qset|\}}\right\}} \\
    & \le \E\left[\sumT \ell_t(\U)\right] + \Eb{\max\left\{ 2 h(\U) \sqrt{\rho K L \revt}, \frac{2K^2 L h(\U)^2}{\max\{1, |\Qset|\}}\right\}},
\end{align*}
which completes the proof of the first statement of Theorem \ref{th:expectation bound}.

Now, in the case where there exists a $\U \in \doma{W}$ such that $\sumT \ell_t(\U) = 0$ for all realizations of the learners' actions, by the guarantee of $\doma{A}$ we have 
\begin{equation*}
    \begin{split}
    \E\left[\sumT \ell_t(\W_t)\right] & = \E\left[\sumT \big(\lhat_t(\W_t) - \lhat_t(\U)\big)\right]\\
    & \le  \inf_{\eta > 0}  \bigg\{ \Eb{\frac{  h(\U)^2}{2\eta}}   + \E\left[\sumT \frac{\eta \vmax}{2}\|\nabla \ell_t(\W_t)\|^2\right]\bigg\} \\
    & \le \inf_{\eta > 0}  \bigg\{ \Eb{\frac{  h(\U)^2}{2\eta}}   + \E\left[\sumT \eta \vmax L \ell_t(\W_t)\right]\bigg\} \\
    & \le \Eb{\vmax L h(\U)^2}   + \half \E\left[\sumT \ell_t(\W_t)\right],
    \end{split}
\end{equation*}
where we used that $\ell_t$ is a regular surrogate loss (in particular equation \eqref{eq:gradient condition}) and plugged in ${\eta = 2 (E[{\vmax}]L)^{-1}}$. After reordering, the above gives us
\begin{align}\label{eq:algAtrick sepa}
    \E\left[\sumT \ell_t(\W_t)\right] & \le 2 \Eb{\vmax L h(\U)^2}.
\end{align}
Now, by using Lemma \ref{lem: upper bound zo loss}, \eqref{eq:vmaxE} and equation \eqref{eq:algAtrick sepa} we have that 
\begin{align*}
    & \Eb{\sumT \zot} \\
    & \leq  \Eb{\sumT \frac{K-1}{K}\ell_t(\W_t)} + \Eb{\sumT \gamma_t} - \Eb{\sumT \lhat_t(\W_t)} + \Eb{\sumT \lhat_t(\W_t)} \\
    & \leq  \Eb{\sumT \gamma_t} - \frac{1}{K}\Eb{\sumT \ell_t(\W_t)} + 2 \E[{\vmax}] L h(\U)^2 \\
    & \leq  2 \Eb{h(\U)\sqrt{L|\Sset\revt}} - \frac{1}{K}\Eb{\sumT \ell_t(\W_t)}  \\
    & ~~ + \Eb{\max\left\{ 2  h(\U)\sqrt{\rho L\revt}, \frac{2 K L h(\U)^2}{\max\{1, |\Qset|\}}\right\}} \\
    & \leq  \Eb{\max\left\{ 4  h(\U)\sqrt{\rho L\revt}, \frac{4 K L h(\U)^2}{\max\{1, |\Qset|\}}\right\}} - \frac{1}{K}\Eb{\sumT \ell_t(\W_t)},
\end{align*}
which completes the proof for the second statement of Theorem \ref{th:expectation bound}.
\end{proof}

\section{Details of Section \ref{sec:high probability} (Bounds that hold with high probability)}\label{app:high prob}
We first provide a Lemma due to \citet{beygelzimer2011contextual} which we use to prove our high-probability bounds. 
\begin{lemma}\label{lem: freedman}
\citep[Theorem 1]{beygelzimer2011contextual}
Let $Z_1, \ldots, Z_T$ be a sequence of real-valued random variables. Suppose that $Z_t \leq B$ and $\E_t[Z_t] = 0$. We have with probability $1 - \delta$
\begin{align*}
    \sumT Z_t \leq \sqrt{3\ln\frac{1}{\delta} \sumT \E_t[Z_t^2]} + 2 B \ln\frac{1}{\delta}.
\end{align*}
\end{lemma}
\begin{theorem}\label{th:high probability}
Under the conditions of Lemma \ref{lem: surrogate gap}, with probability at least $1 - \delta$,  \textproc{Gappletron} with $\gamma = \sqrt{K\rho\big(L h(\U)^2 + 5 \lmax \ln(2/\delta)\big)} $ satisfies: 
\begin{align*}
    \sumT &\zot
\le
    \sumT \ell_t(\U) + (3K+1)\ln\frac{2}{\delta}
\\ &+ \max\left\{5\sqrt{K\rho T\left(Lh(\U)^2 + 5\lmax \ln\frac{1}{\delta'}\right)}, \frac{7 K^2\big(Lh(\U)^2 + 5\lmax \ln(1/\delta')\big)}{\max\{1, |\Qset|\}}\right\}.
\end{align*}
Furthermore, if there exists a $\U \in \doma{W}$ such that $\sumT \ell_t(\U) = 0$
\footnote{Note that $\sumT \ell_t(\U) = 0$, where $\ell_t$ may depend on the learner's randomness, is a weaker condition than standard separability. For example, if some $\U$ satisfies this condition for the standard multiclass hinge loss, then $\U$ satisfies the same condition also for our version of the multiclass hinge, see~(\ref{eq:multiclass hinge}).}, then with probability at least $1-\delta$ \textproc{Gappletron} run with $\gamma = \sqrt{\big((K + 2) \lmax \ln(2/\delta) + L h(\U)^2\big) \rho}$ satisfies:
\begin{align*}
    \sumT \zot
\le
    2 \ln\frac{2}{\delta}
    + 5 &\max\left\{\sqrt{\big((K + 2) \lmax \ln(2/\delta) + L h(\U)^2\big) \rho T}, \right.
\\&\quad
    \left.\frac{2 K}{\max\{1, |\Qset|\}}\big((K + 2) \lmax \ln(2/\delta) + L h(\U)^2\big)\right\}.
\end{align*}
\end{theorem}
\begin{proof}
Before starting, we find a deterministic upper bound on the right-hand side of~(\ref{eq:vmax}) in the proof of Theorem \ref{th:expectation bound}. First, consider $\gamma_T$. By definition it depends on $|\{t: y^*_t \notin \Qset\}|$, which is random, however for any realization we can exploit the trivial bound $|\{t: y^*_t \notin \Qset\}|\le T$ to argue that $\gamma_T \ge \min\left\{\frac 12, \frac{\gamma}{\sqrt{T}}\right\}$. Furthermore, if $\rho > 1$ then $|\Qset| = 0$ which means that if $\half < \frac{\gamma}{\sqrt{T}}$ then $\vmax \leq 2K$. On the other hand, if $\rho = 1$ then $|\Qset| \geq 1$ which means that if $\half < \frac{\gamma}{\sqrt{T}}$ then $\vmax \leq \frac{2K}{|\Qset|}$. With that in mind, we can further bound equation \eqref{eq:vmax}:
\begin{equation}
\label{eq:vmax_det}
    \vmax \leq \max\left\{\frac{\rho}{\gamma_T}, \frac{K}{\max\{1, |\Qset|\}}\right\} \le \max\left\{\frac{\rho\sqrt{T}}{\gamma}, \frac{2K}{\max\{1,|\Qset|\}} \right\}= \Vmax
\end{equation}
As a first step in the actual proof, we study the concentration of the random variables $\zot$ around their means $\zopt$. In order to do so, consider their differences
$z_t = \zot - \zopt$, 
which have zero mean and are bounded in $[-1,1]$. By Lemma \ref{lem: upper bound zo loss} we have
\[
    \E_t\big[z_t^2\big] \le \E_t\big[\zot\big] \le \left(\frac{K-1}{K} \ell_t(\W_t) + \gamma_t\right).
\]
Thus, we can use Lemma \ref{lem: freedman} and, with probability at least $1-\delta'$ we have that
\begin{equation}\label{eq:hp zero-one to surrogate}
    \begin{split}
    \sumT z_t
& \le
    \inf_{\eta > 0}\left\{ \frac{3\ln(1/\delta')}{2\eta} + \frac{\eta}{2}\sumT \left(\frac{K-1}{K} \ell_t(\W_t) + \gamma_t\right)\right\} + 2 \ln\frac{1}{\delta'}
\\&\le
    \frac{3(K-1)\ln(1/\delta')}{4\eta'} + \sumT \left(\frac{\eta'}{K} \ell_t(\W_t) + \frac{\eta'}{K-1}\gamma_t\right) + 2 \ln\frac{1}{\delta'},  
    \end{split}
\end{equation}
where last inequality follows by scaling the argument of the infimum $\eta = \frac{2\eta'}{K-1}$, thus the inequality holds for all $\eta'>0.$
Similarly, we can argue about the concentration of $v_tr_t$ around $r_t$, where $r_t = \ell_t(\U) - \frac{K-\eta'}{K} \ell_t(\W_t)$. Note that $\E_t[v_tr_t-r_t]=0$ and $|v_tr_t - r_t| \leq 2\lmax \Vmax$. Moreover
\[
\E_t[(v_tr_t - r_t)^2] \le \E_t\big[(v_t r_t)^2\big] \le 2\Vmax \lmax |r_t| \le 2\Vmax \lmax\left( \ell_t(\W_t) + \ell_t(\U)\right).
\]
We can finally apply Lemma \ref{lem: freedman} on $v_t r_t - r_t$. Therefore, with probability at least $1-\delta'$ it holds that
\begin{equation}\label{eq: hp vt sur to sur}
    \begin{split}
    \sumT (v_t r_t - r_t)
    & \le \sqrt{6\ln\frac{1}{\delta'} \sumT\Vmax \lmax\left( \ell_t(\W_t) + \ell_t(\U)\right)} + 4 \Vmax \lmax \ln\frac{1}{\delta'} \\
    & \le \frac{3 \Vmax \lmax K }{2\eta'}\ln\frac{1}{\delta'} + \frac{\eta'}{K} \sumT  \big(\ell_t(\W_t) + \ell_t(\U)\big) + 4 \Vmax \lmax \ln\frac{1}{\delta'},
    \end{split}
\end{equation}
where the inequality holds for all $\eta'>0.$

Choosing $\delta'=\frac \delta 2$, we can conclude that both (\ref{eq:hp zero-one to surrogate}) and (\ref{eq: hp vt sur to sur}) hold with probability at least $1-\delta$, for any $\eta'>0$.
The rest of the proof consists then in showing that (\ref{eq:hp zero-one to surrogate}) and (\ref{eq: hp vt sur to sur}) deterministically imply the claimed bound. In particular, we study two different cases, i.e., when $\sumT \ell_t(\U) > \sumT \ell_t(\W_t)$ and its converse.

We first consider $\sumT \ell_t(\U) \le \sumT \ell_t(\W_t)$. By Lemma \ref{lem: surrogate gap} we find that for any $\eta' \in (0, 1]$
\begin{equation*}
    \begin{split}
    \sumT &\zopt -\sumT \gamma_t \\
    &\le \sumT v_t \ell_t(\U) + \inf_{\eta > 0}  \bigg\{ \frac{ h(\U)^2}{2\eta} + \sumT \left(\frac{K-1}{K}\ell_t(\W_t) - \vt \ell_t(\W_t) + \eta \vt^2 L \ell_t(\W_t) \right)\bigg\} \\
    &\le \sumT v_t \ell_t(\U) + \inf_{\eta > 0}  \bigg\{ \frac{ h(\U)^2}{2\eta} + \sumT \left(\frac{K-1}{K}\ell_t(\W_t) - \vt \ell_t(\W_t) + \eta \Vmax \vt L \ell_t(\W_t) \right)\bigg\} \\
    &\le \sumT v_t \ell_t(\U) +  \frac{ K L \Vmax h(\U)^2}{2\eta'} + \sumT \left(\frac{K-1}{K}\ell_t(\W_t) - \frac{K-\eta'}{K} \vt \ell_t(\W_t) \right) \\
    &= \sumT \ell_t(\U)  + \sumT \left(v_t r_t -r_t\right) + \frac{ K L \Vmax h(\U)^2}{2\eta'} + \sumT \frac{\eta'-1}{K}\ell_t(\W_t),
    \end{split}
\end{equation*}
where we have scaled down $\eta' = KL\Vmax\eta$. Substituting in~(\ref{eq: hp vt sur to sur}), we get
\begin{align}
\notag
    \sumT \zopt &\le \sumT \gamma_t + \left(1+\frac{\eta'}K\right)\sumT \ell_t(\U) + \frac{ K L \Vmax h(\U)^2}{2\eta'} \\
    \label{eq:aux1}
    &+\sumT \frac{2\eta'-1}{K}\ell_t(\W_t) + \frac{3 \Vmax \lmax K }{2\eta'}\ln\frac{1}{\delta'} + 4 \Vmax \lmax \ln\frac{1}{\delta'}.
\end{align}
Equations (\ref{eq:hp zero-one to surrogate}) and (\ref{eq:aux1}) are all the ingredients we need to conclude the first case, in fact:
\begin{align*}
    &\sumT \zot - \ell_t(\U)\\
    &\le \sumT \left(\zot - \zopt\right) + \sumT \zopt - \ell_t(\U)\\
    &\le \frac{4\eta'-1}{K} \sumT \ell_t(W_t) + \left(1+\frac{\eta'}{K-1}\right)\sumT \gamma_t + 2 \ln\frac{1}{\delta'}\\
    & + \frac{3(K-1)}{4\eta'} \ln\frac{1}{\delta'} + \frac{LK\Vmax h(\U)^2}{2\eta'}+ 4 \Vmax \lmax  \ln\frac{1}{\delta'} + \frac{3\Vmax\lmax K}{2\eta'}\ln\frac{1}{\delta'}\\
    &\le \frac{5}{4}\sumT \gamma_t+ (3K+1) \ln\frac{1}{\delta'} + 2 \Vmax K\left(Lh(\U)^2 + 5\lmax \ln\frac{1}{\delta'}\right),
\end{align*}
where in the last step make the substitution $\eta'=\frac{1}{4}$.
Now, we have that $\sumT \gamma_t \leq 2 \gamma \sqrt{T}$ and hence we obtain
\begin{equation}\label{eq:hp final 1}
    \begin{split}
    \sumT \zot &- \sumT \ell_t(\U) \\
    &\le 3 \gamma \sqrt{T} + (3K+1) \ln\frac{1}{\delta'} \\ 
    &+ \max\left\{\frac{\rho 2 K \sqrt{T}}{\gamma}, \frac{4K^2}{\max\{1,|\Qset|\}} \right\}
    \left(Lh(\U)^2 + 5\lmax \ln\frac{1}{\delta'}\right) \\
    &\le (3K+1) \ln\frac{1}{\delta'} \\
    &+\max\left\{5\sqrt{K\rho T\left(Lh(\U)^2 + 5\lmax \ln\frac{1}{\delta'}\right)}, \frac{7 K^2\big(Lh(\U)^2 + 5\lmax \ln(1/\delta')\big)}{\max\{1, |\Qset|\}}\right\}.
    \end{split}
\end{equation}
Consider now the second case, where $\sumT \ell_t(\U) > \sumT \ell_t(\W_t)$. We are still assuming~(\ref{eq:hp zero-one to surrogate}) and~(\ref{eq: hp vt sur to sur}) both hold, even though in this case we need only~(\ref{eq:hp zero-one to surrogate}). The $\sumT \zopt$ term is in fact upper bounded using Lemma~\ref{lem: upper bound zo loss}. We have:
\begin{equation}\label{eq:hp final 2}
    \begin{split}
    \sumT& \big(\zot - \ell_t(\U)\big) \\
    &= \sumT (\zot-\zopt) + \sumT \zopt -\sumT \ell_t(\U)\\
    &\le \sqrt{3\ln\frac{1}{\delta'} \sumT \left(\frac{K-1}{K} \ell_t(\W_t) + \gamma_t\right)} + 2 \ln\frac{1}{\delta'} + \sumT \zopt - \sumT \ell_t(\U)  \\
    & \le \sqrt{3\ln\frac{1}{\delta'} \sumT \left(\ell_t(\W_t) + \gamma_t\right)} + 2 \ln\frac{1}{\delta'} +\sumT \frac{K-1}{K} \ell_t(\W_t) + \gamma_t -\sumT \ell_t(\W_t) \\
    & \le \half\inf_{\eta > 0}\left\{\frac{3\ln\frac{1}{\delta'}}{\eta} + \eta \sumT \left(\ell_t(\W_t) + \gamma_t\right)\right\} + 2 \ln\frac{1}{\delta'} +\sumT \frac{K-1}{K} \ell_t(\W_t) + \gamma_t -\sumT \ell_t(\W_t) \\
    & \le \frac{3}{4} K \ln\frac{1}{\delta'} + \frac{5}{2} \sumT \gamma_t + 2 \ln\frac{1}{\delta'}\\
    & \le 3K \ln\frac{1}{\delta'} + 5 \gamma \sqrt{T}\\
    &=3K\ln\frac{1}{\delta'} + \sqrt{K\rho T\left(Lh(\U)^2 + 5\lmax \ln\frac{1}{\delta'}\right)}.
    \end{split}
\end{equation}
The first inequality is due to~(\ref{eq:hp zero-one to surrogate}), while the second one to Lemma \ref{lem: upper bound zo loss}. The third inequality follows from choosing $\eta = \frac{1}{K}$ and finally, the last equality follows by substituting the stated $\gamma.$

In order to prove the second statement, we assume there exists a $\U \in \doma{W}$ such that $\sumT \ell_t(\U) = 0$ for all realizations of the learners' predictions. By the guarantee on $\Aset$ and inequality~\eqref{eq:gradient condition} we have
\begin{equation*}
    \begin{split}
    \sumT \lhat_t(\W_t) & = \sumT \Big(\lhat_t(\W_t) - \lhat_t(\U)\Big)   \\
    & \le  \inf_{\eta > 0}  \bigg\{ \frac{  h(\U)^2}{2\eta} + \frac{\eta}{2} \sumT v_t^2\|\nabla \ell_t(\W_t)\|^2 \bigg\} \\
    & \le  \inf_{\eta > 0}  \bigg\{ \frac{ h(\U)^2}{2\eta} + \eta \Vmax L \sumT v_t \ell_t(\W_t) \bigg\} \\
    & \le \Vmax L h(\U)^2   + \frac{1}{2} \sumT v_t \ell_t(\W_t),
    \end{split}
\end{equation*}
which, after recalling that $\lhat_t(\W_t) = v_t \ell_t(\W_t)$, and reordering, gives us
\begin{equation}\label{eq: hp algA sepa trick}
  \sumT \lhat_t(\W_t) \leq  2\Vmax L h(\U)^2.
\end{equation}
By Lemma \ref{lem: freedman}, we have that with probability at least $1 - \delta'$
\begin{equation}\label{eq: hp sepa vt sur to sur}
\begin{split}
    \sumT \Big(\ell_t(\W_t) - v_t\ell_t(\W_t)\Big) & \le \sqrt{3\lmax\Vmax\ln\frac{1}{\delta'}\sumT\ell_t(\W_t)} + 2\lmax\Vmax\ln\frac{1}{\delta'} \\
    & \le \frac{3 K \lmax\Vmax}{4\eta'}\ln\frac{1}{\delta'} + \frac{\eta'}{K}\sumT\ell_t(\W_t) + 2\lmax\Vmax\ln\frac{1}{\delta'},
\end{split}
\end{equation}
for all $\eta'>0$.
By equation \eqref{eq:hp zero-one to surrogate}, with probability at least $1 - \delta'$ we have that for all $\eta'>0$
\begin{align*}
    \sumT \zot & = \sumT \zopt + \sumT\left(\zot - \zopt\right) \\
    & \le \sumT \zopt + \frac{3(K-1)\ln(1/\delta')}{4\eta'} + 2 \ln(1/\delta') \\
    & ~~ + \sumT \left(\frac{\eta'}{K} \ell_t(\W_t) + \frac{\eta'}{K-1}\gamma_t\right). 
\end{align*}
We continue by using Lemma \ref{lem: upper bound zo loss} to bound $\zopt$:
\begin{align*}
    \sumT \zot \leq &\sumT \frac{K-1}{K} \ell_t(\W_t) + (1 + \frac{\eta'}{K-1})\sumT \gamma_t + \frac{3(K-1)\ln(1/\delta')}{4\eta'} \\
    & + \sumT \frac{\eta'}{K} \ell_t(\W_t)  + 2 \ln(1/\delta') \\
    = & (1 + \frac{\eta'}{K-1})\sumT \gamma_t + \frac{3(K-1)\ln(1/\delta')}{4\eta'} + \frac{\eta' - 1}{K} \ell_t(\W_t)  \\
    & + 2 \ln(1/\delta')  + \sumT v_t\ell_t(\W_t) + \sumT \left(\ell_t(\W_t) - v_t\ell_t(\W_t)\right).
\end{align*}
By equation \eqref{eq: hp sepa vt sur to sur} and the union bound, with probability at least $1 - 2 \delta'$:
\begin{align*}
    \sumT \zot \leq &\sumT \frac{K-1}{K} \ell_t(\W_t) + (1 + \frac{\eta'}{K-1})\sumT \gamma_t + \frac{3(K-1)\ln(1/\delta')}{4\eta'} \\
    & + \sumT \frac{\eta'}{K} \ell_t(\W_t)  + 2 \ln(1/\delta') + 2\lmax\Vmax\ln(1/\delta') \\
    = & (1 + \frac{\eta'}{K-1})\sumT \gamma_t + \frac{3 K \lmax\Vmax\ln(1/\delta')}{4\eta'} + \frac{2\eta' - 1}{K} \sumT \ell_t(\W_t)  \\
    & + (\frac{3}{4\eta'} (K-1) + \half) \ln(1/\delta')  + \sumT v_t\ell_t(\W_t) + 2\lmax\Vmax\ln(1/\delta').
\end{align*}
We use equation \eqref{eq: hp algA sepa trick}, $\sumT \gamma_t \leq 2\gamma\sqrt{T}$, set $\eta' = \half$, set $\delta' = \half \delta$, and the definition of $\Vmax$ in equation \eqref{eq:vmax_det} to continue: 
\begin{align*}
    \sumT \zot & \le \frac{3}{2}\sumT \gamma_t + 2 \Vmax ((K + 2) \lmax \ln(2/\delta) + L h(\U)^2) + 2 \ln(2/\delta) \\
    & \le 2 \ln(2/\delta) + 5 \max\bigg\{\sqrt{((K + 2) \lmax \ln(2/\delta) + L h(\U)^2) \rho T}, \\
    &  ~~~~ \frac{2 K}{\max\{1, |\Qset|\}}((K + 2) \lmax \ln(2/\delta) + L h(\U)^2)\bigg\},
\end{align*}
which holds with probability at least $1 - \delta$ and completes the proof of the second statement of Theorem \ref{th:high probability}. 
\end{proof}
We now restate Theorem \ref{th:fullinfo hp}, after which we prove it. 
\thfullinfohp*
\begin{proof}[Proof of Theorem \ref{th:fullinfo hp}]
Denote by $z_t = \left(\zot - \zopt\right)$. By Lemma \ref{lem: freedman}, with probability at least $1 - \delta$ we have that \begin{align*}
    \sumT z_t \le \sqrt{3 \ln(1/\delta) \sumT \E_t\left[z_t^2\right]} + 2 \ln(1/\delta)
    = \inf_{\eta > 0} \left\{\frac{3 \ln(1/\delta)}{2 \eta} + \frac{\eta}{2}\sumT \E_t\left[z_t^2\right]\right\} + 2 \ln(1/\delta).
\end{align*}
Since the variance is bounded by the second moment, we have that 
\begin{equation*}
    \E_t\left[z_t^2\right] \leq \E_t\left[\zot\right] \leq \frac{K-1}{K} \ell_t(\W_t),
\end{equation*}
where the last inequality is due to Lemma \ref{lem: upper bound zo loss}. By applying Lemma \ref{lem: surrogate gap}, and recalling that $\gamma_t=0$ and $v_t = 1$ because we are in the full information setting, we find that with probability at least $1 - \delta$
\begin{align*}
 \sumT \zot & \le \sumT \ell_t(\U)  + \inf_{\eta > 0} \left\{ \frac{h(\U)^2}{2\eta} + \sumT \left(\eta L - \frac{1}{2K}\right)\ell_t(\W_t)\right\} \\
 & + \inf_{\eta > 0} \left\{\frac{3 \ln(1/\delta)}{2 \eta} + \sumT \left(\frac{\eta}{2}\frac{K-1}{K} - \frac{1}{2K} \right)\ell_t(\W_t)\right\} + 2 \ln(1/\delta) \\
 & \le \sumT \ell_t(\U) + K L h(\U)^2 + \left(\frac{3}{2}K+\frac{1}{2}\right)\ln(1/\delta),
\end{align*}
which completes the proof in the non-separable case. In the separable case, when there exists a $\U \in \doma{W}$ such that $\sumT \ell_t(\U) = 0$, we can use Lemma~\ref{lem: upper bound zo loss} to show that, with probability at least $1 - \delta$
\begin{align*}
    \sumT \zot & \le \inf_{\eta > 0} \left\{\frac{3 \ln(1/\delta)}{2 \eta} + \sumT \frac{\eta}{2}\frac{K-1}{K}\ell_t(\W_t)\right\} + 2 \ln(1/\delta) + \sumT \zopt \\
    & \le \inf_{\eta > 0} \left\{\frac{3 \ln(1/\delta)}{2 \eta} + \sumT \frac{\eta}{2}\frac{K-1}{K}\ell_t(\W_t)\right\} + 2 \ln(1/\delta) + \frac{K-1}{K} \sumT \ell_t(\W_t) \\
    & \le  \frac{11}{4} \ln(1/\delta) + 2 \sumT \ell_t(\W_t) \\
    & \le  \frac{11}{4} \ln(1/\delta) + 4 L h(\U)^2,
\end{align*}
where in the last inequality we used equation \eqref{eq: hp algA sepa trick} with $\Vmax = 1$.
\end{proof}

\section{Details of Section \ref{sec:lower bounds} (Lower Bounds)} \label{app:lower bounds}

\begin{theorem}\label{th:apple lower bound}
In the spam filtering classification setting with smooth hinge loss, the surrogate regret of any (possibly randomized) algorithm satisfies
\[
    \Eb{\sumT \id[y_t' \neq y_t]} = \sumT \ell_t(\bUhat) + \Omega\big(\!\sqrt{T}\big)
\]
for some label sequence $y_1,\ldots,y_T \in \{1,2\}$, for some sequence of feature vectors $\x_t$ such that $\norm{\x_t}_2 = \sqrt{2}$ for all $t$, and for some $\bUhat$ such that $\norm{\bUhat}_2 \le \sqrt{5}$.
\end{theorem}
\begin{proof}
We adapt an argument from \citet[Lemma~3]{daniely2015strongly}.
Let $R = \big\lfloor\sqrt{T/2}\big\rfloor$ and divide the $T$ rounds in $2R$ segments $T_1,\ldots,T_{2R}$ of size $T/(2R)$ each (without loss of generality, assume that $2R$ divides $T$). For each segment $T_i$, define the components $x_{t,z}$ of the feature vectors $\x_t$ at rounds $t \in T_i$ as follows: $x_{t,z} = 1$ if $z \in \{1,i+1\}$ and $x_{t,z} = 0$ otherwise.

Fix an algorithm $A$ and assume $y_t = 1$ for all $t$. Denote by $N_2$ the rounds in which $A$ predicts $2$. If $\Eb{|N_2|} \ge R$, then $A$ makes more than $R = \Omega\big(\!\sqrt{T}\big)$ mistakes and we are done because
\begin{align*}
    \sumT \ell_t(\bUhat)
=
	\sumT \max\Big\{\big(1 - \hat{U}^1_1 + \hat{U}^2_1\big)^2,\, 0\Big\} = 0
\end{align*}
for $\bUhat$ defined as follows: $\Uhat^1_1 = 1$, $\Uhat^1_z = 0$ for $z > 1$, and $\Uhat^2_z = 0$ for all $z$. 
Consider then $\Eb{|N_2|} \le R$. Since there are $2R$ segments, we must have that $\Eb{|N_2 \cap T_j|} \leq \frac{1}{2}$ for some $j \in [2R]$, because otherwise $\Eb{|N_2|} = \Eb{\sum_{i=1}^{2R}|N_2 \cap T_i|} > R$. Now, by Markov's inequality we have that 
$
    \bP(|N_2 \cap T_j| \geq 1) \leq \frac{1}{2}
$
which means that $A$ does not predict $2$ in segment $j$ with probability at least $\frac{1}{2}$.

Now set $y_t=2$ for all $t \in T_j$. Since we are in the spam filtering setting, if label $2$ is never predicted in segment $j$, $A$ cannot detect that the label has changed, and so it makes a mistake on each round in $T_j$, which has length $T/(2R)$. Hence
\begin{align*}
	\Eb{\sumT \id[y_t' \neq y_t]}
\ge
	\frac{T}{2R}\bP\left(\sumT \id[y_t' \neq y_t] \ge \frac{T}{2R}\right)
\ge
	\frac{T}{4R}
=
	\Omega\big(\!\sqrt{T}\big)
\end{align*}
Define a new comparator $\bUhat$ as follows: $\Uhat^1_z = 1$ if $z = 1$ and $\Uhat^1_z = 0$ otherwise, and $\Uhat^2_z = 2$ if $z = j+1$ and $\Uhat^2_z = 0$ otherwise. For the same sequence of labels $y_t$, we have that
\begin{align*}
    \sumT \ell_t(\bUhat)
=
	\!\sum_{t \in T_j} \max\Big\{\big(1 - \Uhat^2_1 - \Uhat^2_{j+1} + \Uhat^1_1 + \Uhat^1_{j+1}\big)^2, 0\Big\}
+
	\!\!\!\!\!\sum_{t \in [T] \setminus T_j}\!\!\!\!\! \max\Big\{\big(1 - \Uhat^1_1 + \Uhat^2_1\big)^2, 0\Big\}
\end{align*}
where the sums in the right-hand side are easily seen to be zero.
This concludes the proof.
\end{proof}
\begin{theorem}\label{th:full info lower bound}
Consider the full information setting with smooth hinge loss. For any integer $B\ge 2$, the surrogate regret of any (possibly randomized) algorithm satisfies
\[
    \Eb{\sumT \id[y_t' \neq y_t]} \ge \min_{\U\in\doma{W}}\sumT \ell_t(\U) + (B^2-1)(K-1) + \frac{K-1}{K}
\]
for some label sequence $y_1,\ldots,y_T \in [K]$ and for some sequence of feature vectors $\x_t$ such that $\norm{\x_t}_2 = 1$ for all $t$, where $\doma{W} = \theset{\W}{\norm{\W} \le B}$.
\end{theorem}
\begin{proof}
We sample the labels $y_t$ uniformly at random for the first $M+1$ rounds, where $M = (B^2-1)K^2$. Then we set $y_t = y_{M+1}$ for each $t > M+1$. The feature vectors $\x_t$ have $M+1$ components. For $t=1,\ldots,M$ we set the components $x_{t,z}$ of the feature vectors $\x_t$ as $x_{t,t} = 1$ and $x_{t,z} = 0$ for $z \neq t$. For each $t \ge M + 1$, we set $x_{t,M+1} = 1$ and $x_{t,i} = 0$ for all $i=1,\ldots,M$.
We now define a comparator $\bUhat$ as follows. For each $z=1,\ldots,M$ we set $\Uhat^y_z = \frac{1}{K}$ for $y=y_t$ and $\Uhat^y_z = 0$ otherwise. Then we set $\Uhat^y_{M+1} = 1$ if $y = y_{M+1}$ and $\Uhat^y_{M+1} = 0$ otherwise.
Note that, deterministically, $\norm{\bUhat}_2^2 = 1 + \sum_{t=1}^M \big(U_t^{y_t}\big)^2 = 1 + \frac{M}{K^2} = B^2$.
Now fix any (possibly randomized) algorithm $A$.
With these choices, in the first $M$ rounds we have
\begin{align*}
	\Eb{\sum_{t = 1}^M \id[y_t' \neq y_t]} - \ell_t(\hat{\U})
=
	M\frac{K-1}{K} - M\left(1 - \frac{1}{K}\right)^2 = M\left(\frac{1}{K} - \frac{1}{K^2}\right)~.
\end{align*} 
In the next $T-M$ rounds we have
\begin{align*}
    \Eb{\sum_{t=M+1}^T \id[y_t' \neq y_t]} \ge \frac{K-1}{K}
\qquad\text{and}\qquad
	\sum_{t=M+1}^T \ell_t(\bUhat) = 0~.
\end{align*}
The above expectations are both with respect to the random draw of $y_1,\ldots,y_{M+1}$ and to $A$'s internal randomization. This implies that there exists a sequence $y_1,\ldots,y_{M+1}$ such that the two above bounds hold in expectation with respect to $A$'s internal randomization. Putting the two bounds together concludes the proof.
\end{proof}

\section{Details of Section \ref{sec:experiments} (Experiments)}\label{app:experiments}

\begin{figure}
    \centering
    \includegraphics[width = \textwidth]{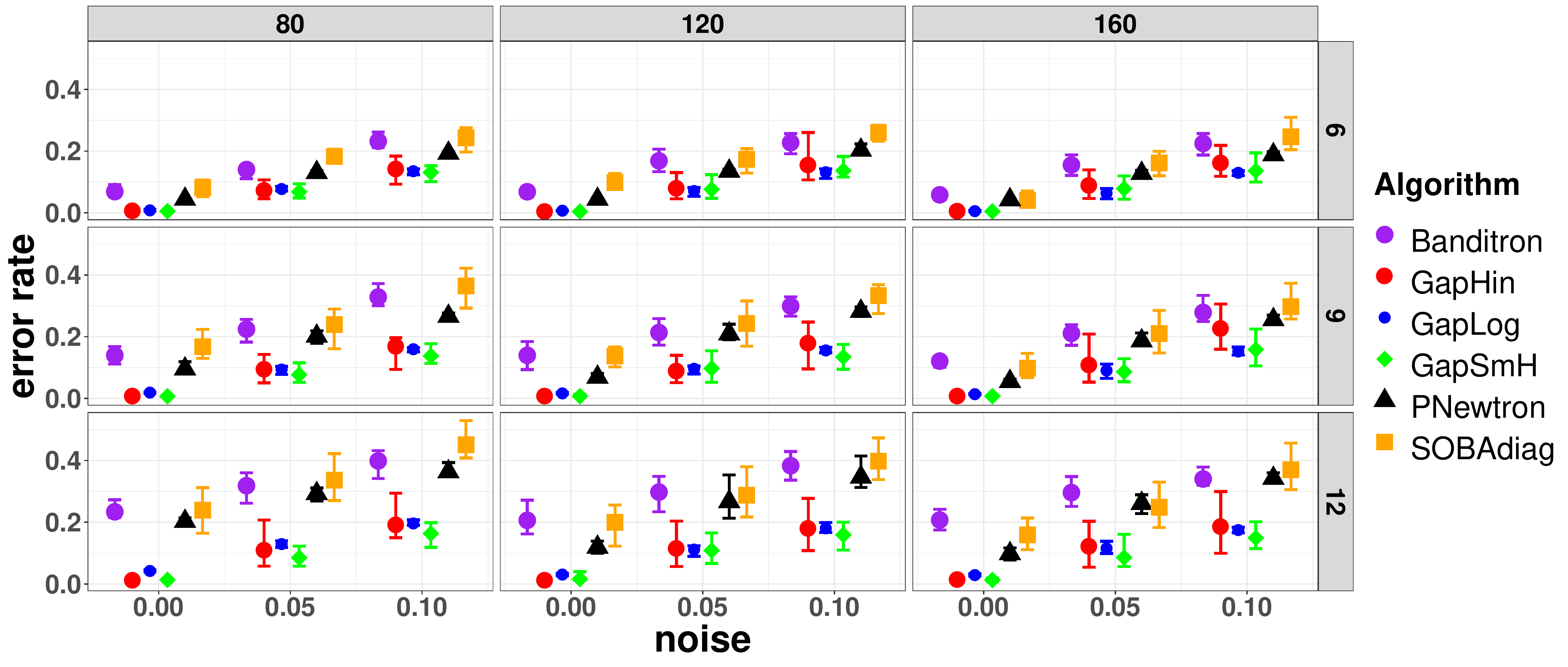}
    \caption{Results of the synthetic experiments for the bandit setting. The parameters of algorithms are set to 1, except for $T$. The rows are the different values for $K$ and the columns are the different values for $d$. The whiskers represent the minimum and maximum error rates of the ten repetitions.}
    \label{fig:SynOnlyT}
\end{figure}

\begin{figure}
    \centering
    \includegraphics[width = \textwidth]{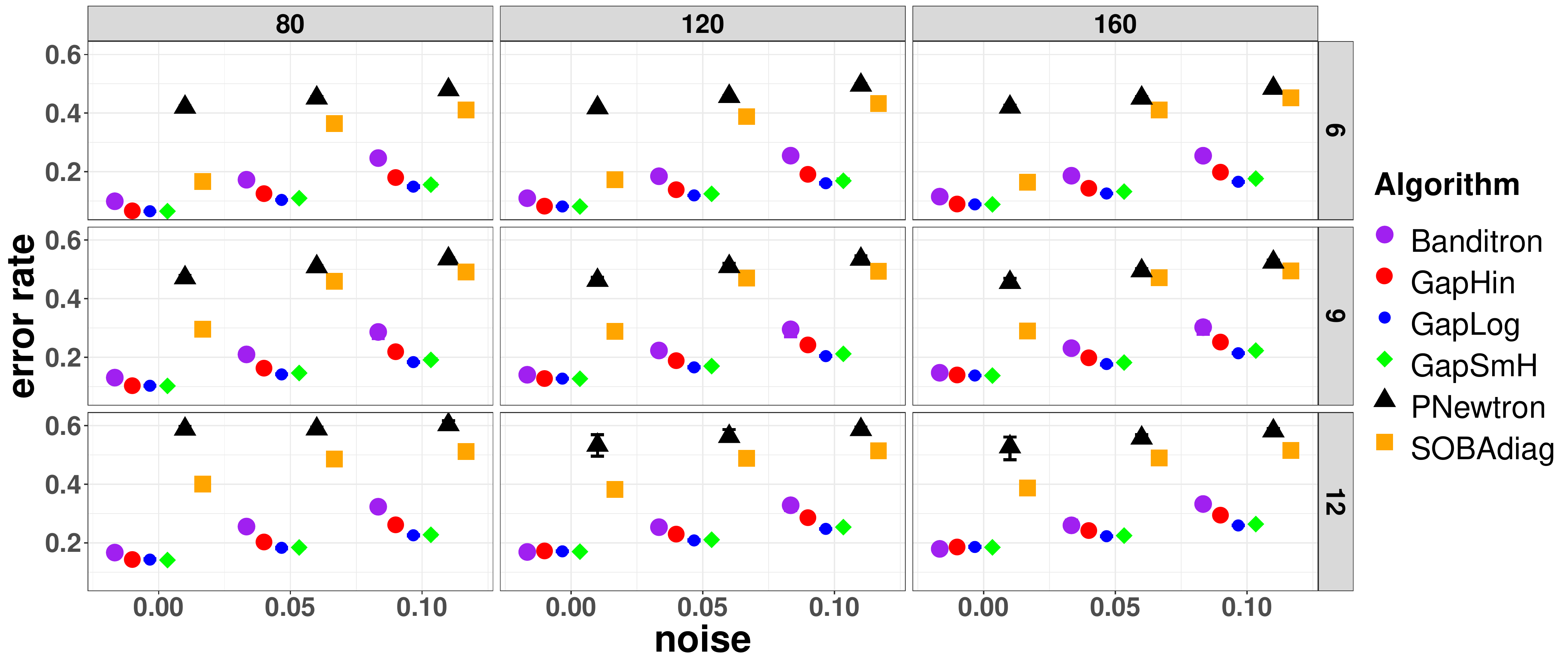}
    \caption{Results of the synthetic experiments for the bandit setting with theoretical tuning. The rows are the different values for $K$ and the columns are the different values for $d$. The whiskers represent the minimum and maximum error rates of the ten repetitions.}
    \label{fig:SynTheo}
\end{figure}

\begin{figure}
    \centering
    \includegraphics[width = \textwidth]{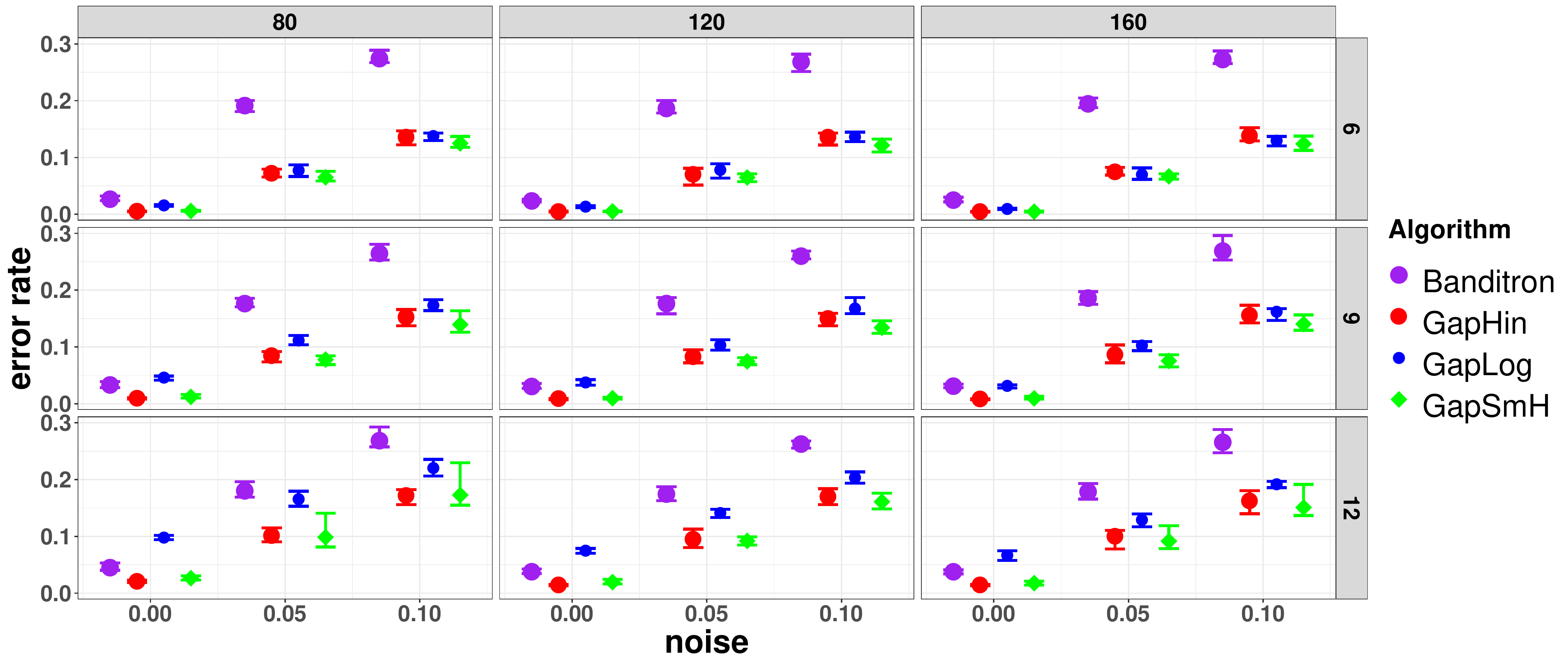}
    \caption{Results of the synthetic experiments for multiclass spam filtering. The parameters of algorithms are set to 1, except for $T$. The rows are the different values for $K$ and the columns are the different values for $d$. The whiskers represent the minimum and maximum error rates of the ten repetitions.}
    \label{fig:SynRAonlyT}
\end{figure}

\begin{figure}
    \centering
    \includegraphics[width = \textwidth]{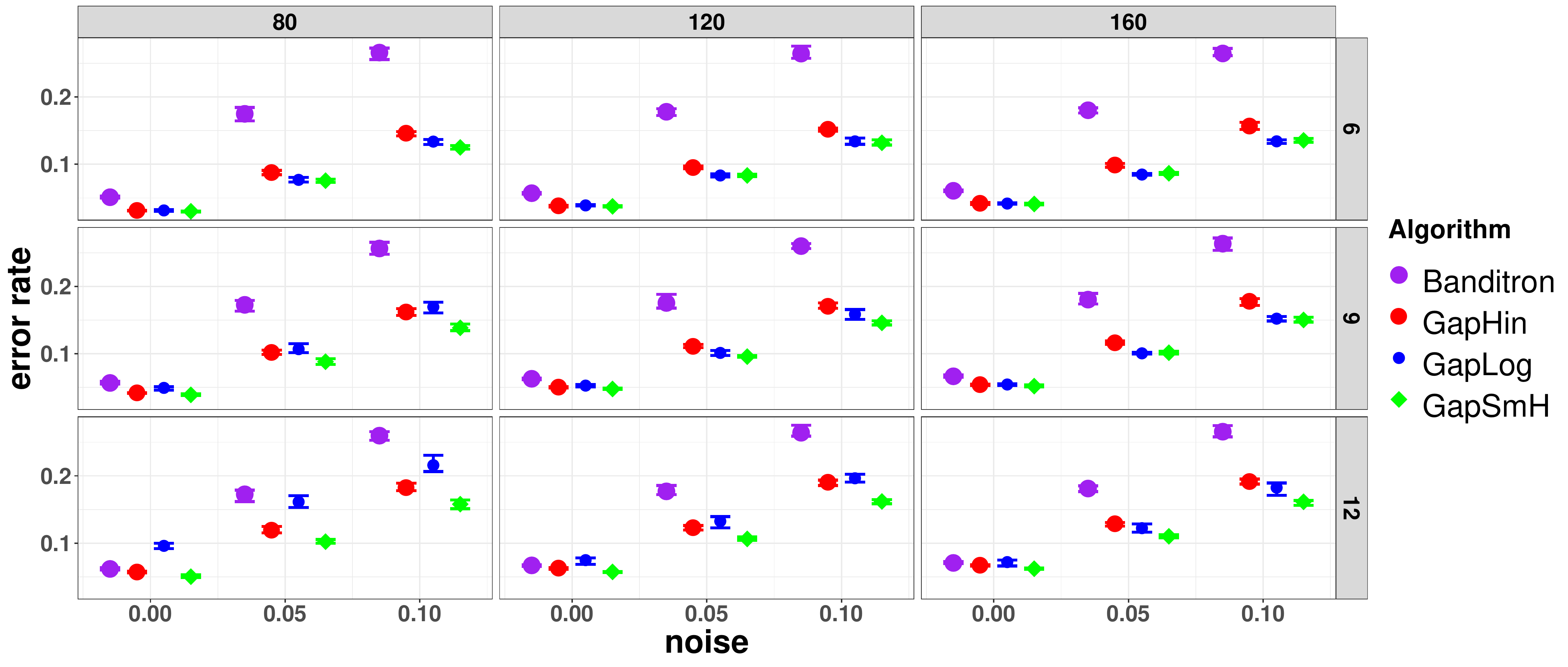}
    \caption{Results of the synthetic experiments for multiclass spam filtering with theoretical tuning. The rows are the different values for $K$ and the columns are the different values for $d$. The whiskers represent the minimum and maximum error rates of the ten repetitions.}
    \label{fig:SynRAtheo}
\end{figure}

\begin{figure}
    \centering
    \includegraphics[width = \textwidth]{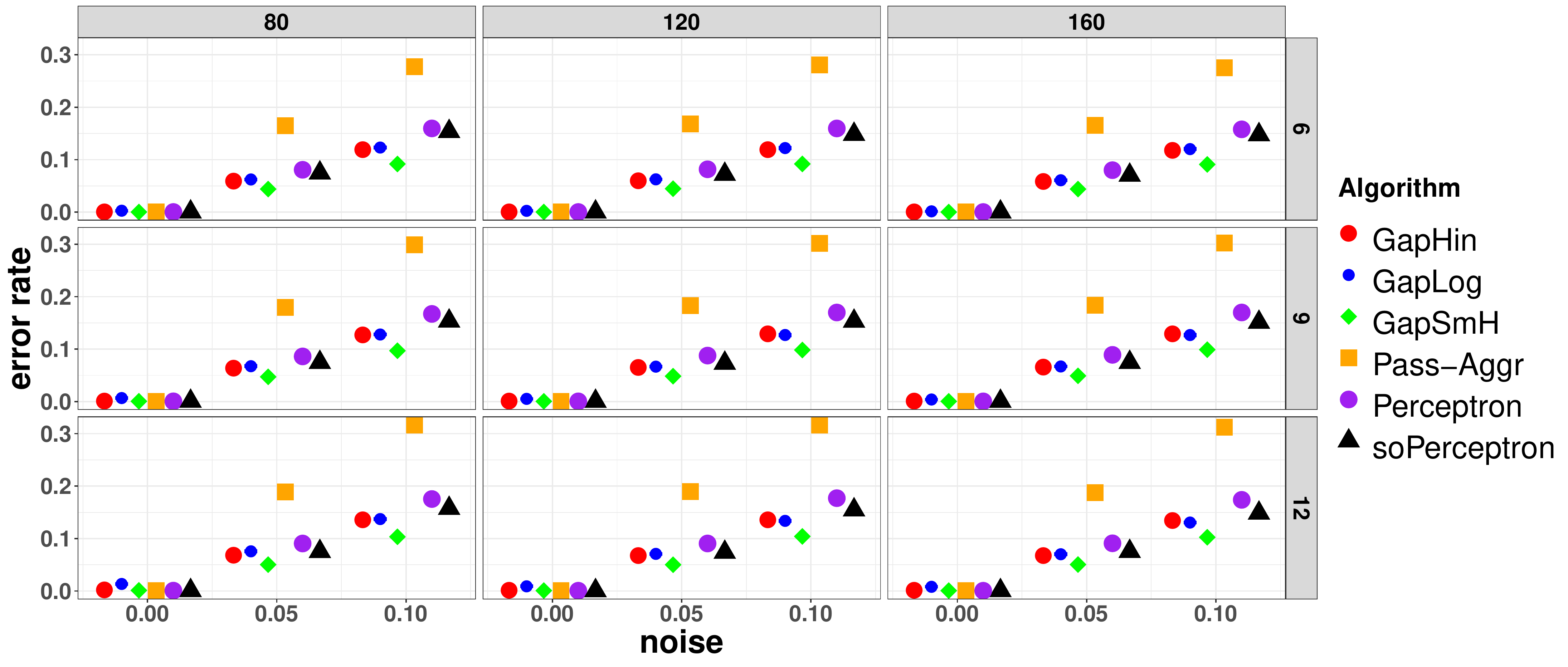}
    \caption{Results of the synthetic experiments for the full information setting. The rows are the different values for $K$ and the columns are the different values for $d$. The whiskers represent the minimum and maximum error rates of the ten repetitions.}
    \label{fig:SynFull}
\end{figure}

\begin{figure}
    \centering
    \includegraphics[width = \textwidth]{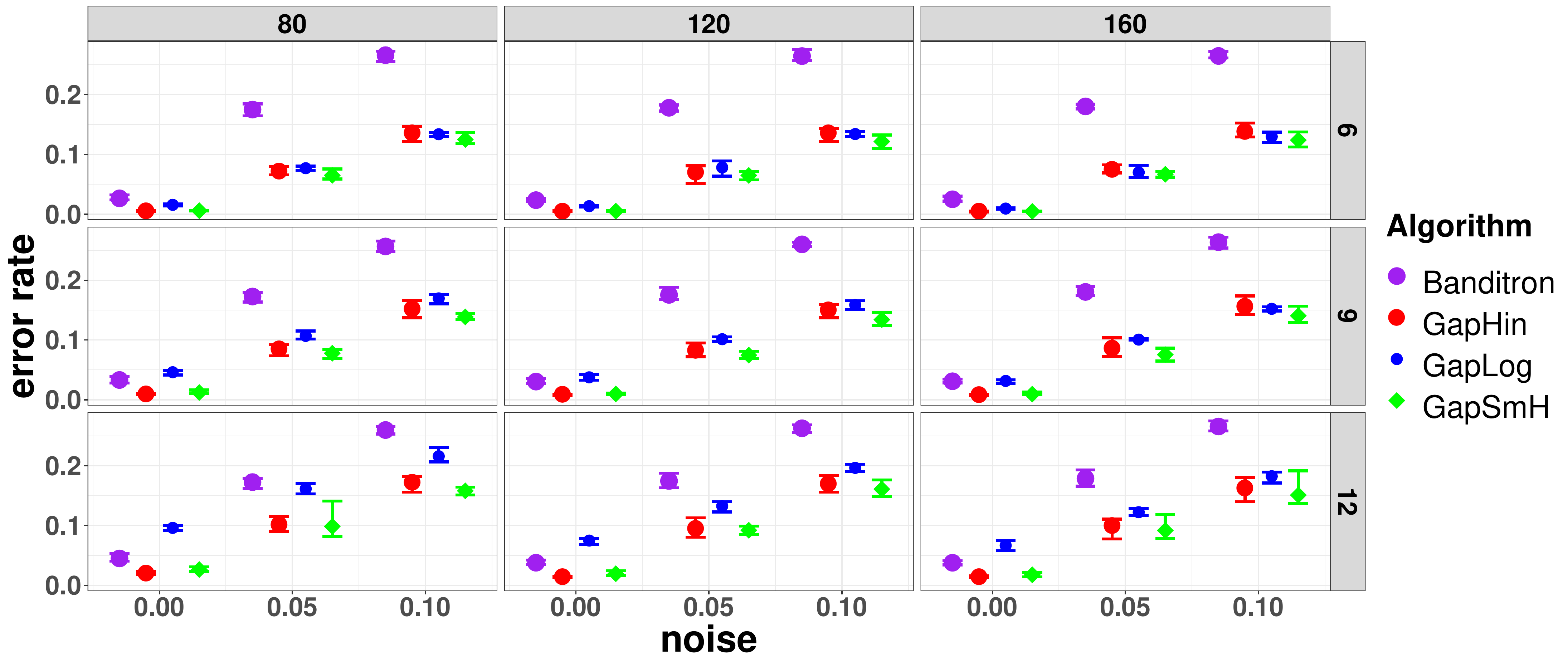}
    \caption{Results of the synthetic experiments for the multiclass spam filtering. The plot shows the best results of algorithms with parameters suggested by theory, or tuned with all parameters set to 1, except for $T$. The rows are the different values for $K$ and the columns are the different values for $d$. The whiskers represent the minimum and maximum error rates of the ten repetitions.}
    \label{fig:SynbestRA}
\end{figure}

The feature vectors $\x_t \in \{0, 1\}^d$ and class labels are generated as follows. For each class we reserve the first $10 d'$ bits to generate ``keywords''. For each class, $1 d'$ to $5 d'$ of these bits are randomly turned on to represent the keywords for that class. The remaining $30 d'$ bits are reserved for unrelated words, of which $5 d'$ are randomly turned on. For each $t$ we select a class uniform at random and set $\x_t$ to be the feature vector described above. Then, with probability 0, 0.05, or 0.1, we replace the class with a uniformly at random chosen class. We varied between 6, 9, or 12 classes and varied $d' \in [2, 3 ,4]$. In the multiclass spam filtering setting we fixed $\Qset = \{1\}$, i.e., querying $y_t$ corresponds to predicting label $1$.

As suggested by \citet{hazan2011newtron}, we tuned PNewtron with $\alpha = 10$ and chose the unit ball as domain. For SOBAdiag, we used the adaptive tuning for the exploration rate in the experiment with theoretical tuning and used a fixed exploration rate in the experiment with tuning based only on $T$.

For the experiments in the partial information setting we tuned the algorithms according to what theory suggests for the worst case. Additionally, we also ran experiments with all parameters set to 1, except for $T$. Initially we only tuned the algorithms with theoretical tuning, but we found that in the bandit setting two of the algorithms we compare with did not have satisfactory performance. All parameters based on (an upper bound on) $\|\U\|$ we set to $1$ as to not advantage or disadvantage algorithms that did not use tuning based on $\|\U\|$. All experiments involving randomness due to the algorithms we repeated ten times. 

All experiments were run on a system with 8GB of ram, an Intel i5-6300U CPU, and in python 3.8.5 on a Windows 10 operating system.  The results of the experiments are summarized in Figures \ref{fig:SynOnlyT}, \ref{fig:SynTheo}, \ref{fig:SynRAonlyT}, \ref{fig:SynRAtheo}, \ref{fig:SynFull}, and \ref{fig:SynbestRA}. We also ran experiments in for the label efficient graph comparing \textproc{Gappletron} with the label efficient \textproc{Perceptron} of \citet{cesa2006worst}. However, as they assume that labels come without a cost it was not clear how to tune their algorithm. We tried several parameter values which still guarantees sublinear regret in $T$. However, with none of the choices of parameters the label efficient \textproc{Perceptron} performed as well as \textproc{Gappletron}, so we choose not to report the results of these experiments. 

In the experiments with bandit feedback, as we mentioned, Figure \ref{fig:SynTheo} shows that with theoretical tuning both PNewtron and SOBAdiag performed poorly. We suspect this is due to the tuning with $d$, as when we do not tune with $d$ the performance of these algorithms greatly improved (see Figure \ref{fig:SynOnlyT}). The error rate of \textproc{Banditron} was the lowest with theoretical tuning in roughly half of the experiments. For GapLog and GapSmH the performance also improved when only tuning with $T$, especially in experiments with no noise. GapHin became more unstable when tuning only with $T$, as can been seen from the spread of the results. We suspect this is due to the fact that with the hinge loss, \textproc{Gappletron} explores less than with the smooth hinge loss and the logistic loss. Note that with the smooth hinge loss \textproc{Gappletron} explores less than with the logistic loss, which also seems to become apparent from the range of performance of these two versions of \textproc{Gappletron}. With theoretical tuning, in low noise settings Banditron is on par with the performance of all versions of \textproc{Gappletron}, but with high noise GapLog and GapSmH outperform all other algorithms. With tuning that only depends on $T$, GapLog and GapSmH strictly outperform all other algorithms. Figure \ref{fig:Synbestbandit} contains the results for the best version of each bandit algorithm, which shows that GapLog and GapSmH outperform all other algorithms.

In the multiclass spam filtering setting we compared  \textproc{Gappletron} with the importance weighted version of Banditron, which explored the revealing action with probability $\max\{\half, (X^2/T)^{1/3}\}$ or with probability $\max\{\half, (1/T)^{1/3}\}$, where the former is the theoretical tuning and $\|\x_t\|_2 \leq X$.

In multiclass spam filtering with theoretical tuning (Figure \ref{fig:SynRAtheo}), \textproc{Banditron} had the lowest error rate in the no-noise experiments. For experiments with noise we see that as $K$ increases the performance of \textproc{Gappletron} deteriorates compared to the performance of \textproc{Banditron}. We suspect this is due the $\sqrt{K}$ in the exploration of \textproc{Gappletron}, which does not appear in the exploration of \textproc{Banditron}\footnote{Although no bound exists for this algorithm in literature, one can adapt the proof of \citet{kakade2008efficient} to prove a $O((X\|\U\|)^{1/3}T^{2/3})$ surrogate regret bound.}. In Figure \ref{fig:SynRAonlyT} we can see that with tuning based solely on $T$, the spread of the algorithms seems to increase, as was the case in the bandit setting. Either GapHin or GapSmH had the lowest error rate in these experiments, which is also true when comparing the algorithms across the tuning for the exploration rate (Figure \ref{fig:SynbestRA}). The performance of GapLog get worse as $K$ increases, is was the case in the full information setting. We suspect this is due the fact that GapLog explores more than GapHin and GapSmH. While in the bandit setting extra exploration gives additional information, in multiclass spam filtering it does not provide additional information and it only leads to making more mistakes. 

In the full information setting we compare \textproc{Gappletron} with the diagonal version of the second-order Perceptron, soPerceptron \citep{cesa2005second}, the multiclass Perceptron, and the passive-aggressive version of the multiclass perceptron \citep{crammer2006online}. In Figure \ref{fig:SynFull}, we can see that if there is no label noise, essentially all algorithms find the separating hyperplane. Note that GapLog has the worst performance in this case. This is due to the fact that with the logistic loss, \textproc{Gappletron} never stops with playing at random, leading to sometimes unnecessarily playing the wrong action. We also see this behavior in experiments with label noise, where GapLog performs worse than the other versions of \textproc{Gappletron}, although its performance in still either on par with or better than the non-\textproc{Gappletron} algorithms in these experiments. Overall, in the full information experiments GapSmH appears to have the best performance.

\end{document}